\documentclass[english,twoside,11pt]{article}

\usepackage{jmlr2e}
\usepackage{amsmath}
\usepackage{url}
\usepackage{sidecap}
\usepackage{subfigure}

\usepackage{graphicx} 
\usepackage{subfigure} 

\usepackage{natbib}

\usepackage{algorithm}
\usepackage{algorithmic}
\usepackage{caption}

\usepackage{hyperref}

\newcommand{\sigmoid}{\mbox{sigmoid}}

\usepackage[utf8]{inputenc}
\usepackage{booktabs}

\usepackage{ownstyles}

\newcommand{\svs}[1]{}
\newcommand{\vs}[1]{}

\newcommand{\otherf}{\phi}
\newcommand{\calP}{P}

\newtheorem{innercustomthm}{Theorem}
\newenvironment{customthm}[1]
  {\renewcommand\theinnercustomthm{#1}\innercustomthm}
  {\endinnercustomthm}

\newif\ifLongversion
\Longversiontrue

\usepackage{parskip}
\usepackage{babel}

\newcommand\blfootnote[1]{%
  \begingroup
  \renewcommand\thefootnote{}\footnote{#1}%
  \addtocounter{footnote}{-1}%
  \endgroup
}

\begin{document}

\firstpageno{1}

\title{GSNs: Generative Stochastic Networks}


\author{
\noindent  Guillaume Alain$^{*+}$, Yoshua Bengio$^{*+}$, Li Yao$^*$, 
 Jason Yosinski$^\dagger$, \'Eric Thibodeau-Laufer$^*$, Saizheng Zhang$^*$ and Pascal Vincent$^*$\\
\\
$^*$ Department of Computer Science and Operations Research\\
University of Montreal\\
Montreal, H3C 3J7, Quebec, Canada\\
\\
$^\dagger$ Department of Computer Science, Cornell University\\
}

\editor{  } 

\maketitle

\begin{abstract}
We introduce a novel training principle for generative probabilistic models that is an
alternative to maximum likelihood. The proposed Generative Stochastic
Networks (GSN) framework generalizes Denoising Auto-Encoders (DAE)
and is based on learning the transition operator of a
Markov chain whose stationary distribution estimates the data distribution.
The transition distribution is a conditional distribution that generally
involves a small move, so it has fewer dominant modes and is unimodal in the
limit of small moves. This simplifies the learning problem, making it less like
density estimation and more akin to supervised function approximation, with gradients that can be
obtained by backprop.
The theorems provided here provide a probabilistic interpretation for denoising autoencoders
and generalize them; seen in the context of this framework, auto-encoders that learn with injected noise
are a special case of GSNs and can be interpreted as generative models.
The theorems also provide an
interesting justification for dependency networks and generalized
pseudolikelihood and define an appropriate joint distribution and
sampling mechanism, even when the conditionals are not consistent. GSNs
can be used with missing inputs and can be used to sample subsets of
variables given the rest. Experiments validating these theoretical results
are conducted on both synthetic datasets and image datasets. The experiments employ a particular
architecture that mimics the Deep Boltzmann Machine Gibbs sampler but
that allows training to proceed with backprop through a recurrent neural
network with noise injected inside and without the need for layerwise
pretraining.
\end{abstract}

\blfootnote{$^+$The first two authors had an equally important contribution}

\section{Introduction}
\svs{2}

\begin{figure}[htpb]
\vs{3}
\centering
\includegraphics[width=0.8\linewidth]{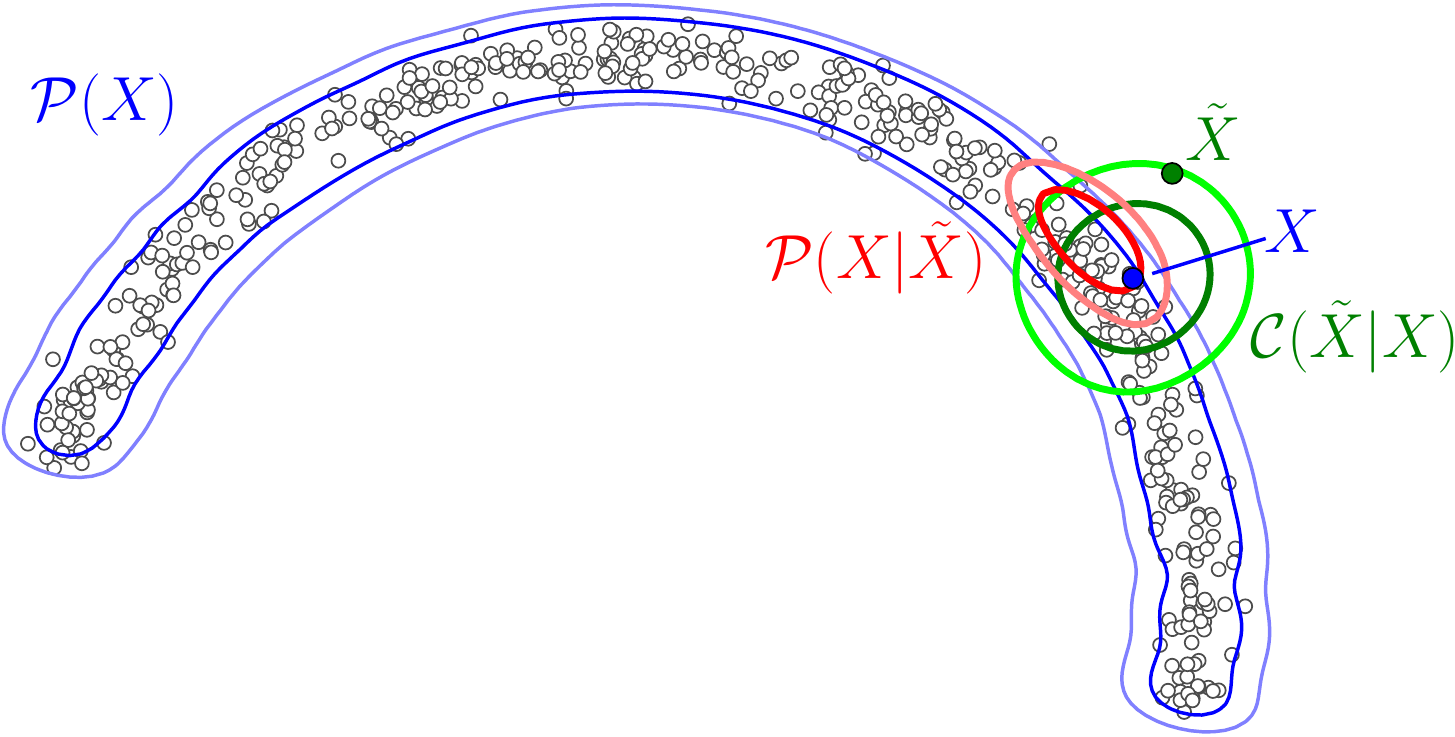}
\includegraphics[width=0.8\linewidth]{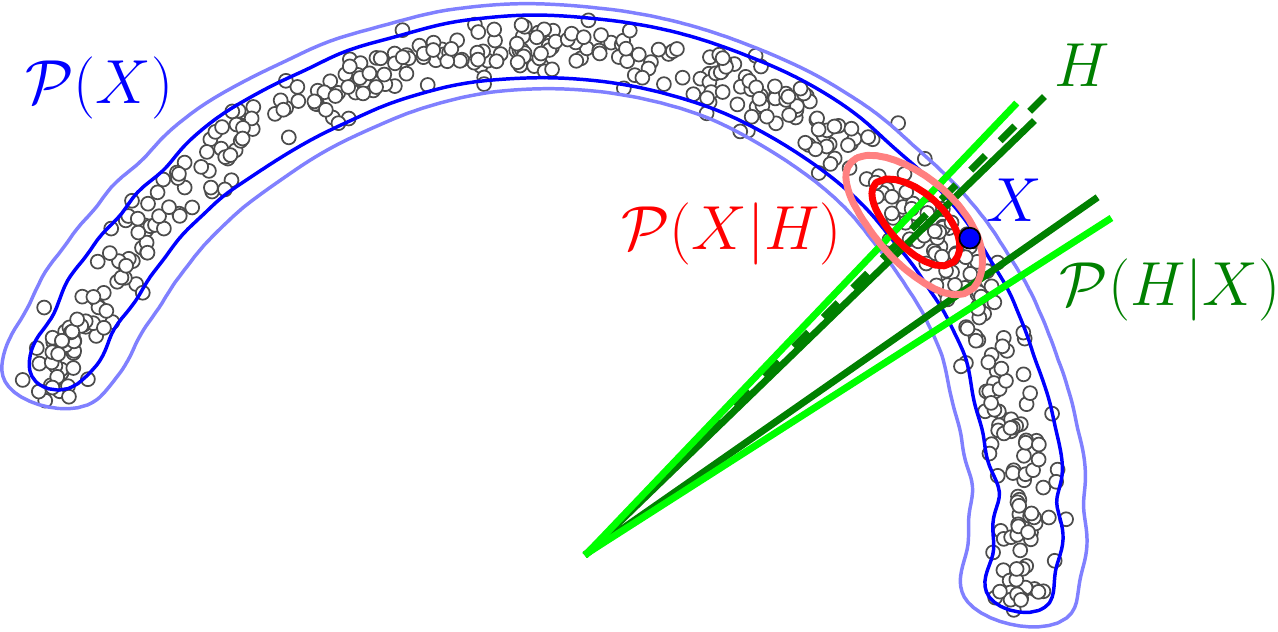}
\vs{3}
\caption{
{\em Top:} A denoising auto-encoder defines an estimated Markov chain where the transition 
operator first samples a corrupted $\tilde{X}$ from ${\cal C}(\tilde{X}|X)$ and then samples a reconstruction from $P_\theta(X|\tilde{X})$, which is trained to estimate the ground truth 
${\calP}(X|\tilde{X})$. Note how for any given $\tilde{X}$,
${\calP}(X|\tilde{X})$ is a much simpler (roughly unimodal) distribution than the
ground truth ${\calP}(X)$ and its partition function is thus easier to approximate.
{\em Bottom:} More generally, a GSN allows the use of arbitrary 
latent variables $H$ in addition to $X$, with the Markov chain state 
(and mixing) involving both $X$ and $H$. Here $H$ is the angle about the origin.
The GSN inherits the benefit of a simpler conditional and adds 
latent variables, which allow more powerful deep representations in which mixing is easier~\citep{Bengio-et-al-ICML2013}.
}
\label{fig:data_px}
\vs{3}
\end{figure}

Research in deep learning~(see \citet{Bengio-2009-book}
and~\citet{Bengio-Courville-Vincent-TPAMI2013} for reviews) grew from
breakthroughs in unsupervised learning of representations, based
mostly on the Restricted Boltzmann Machine (RBM)~\citep{Hinton06},
auto-encoder
variants~\citep{Bengio-nips-2006-small,VincentPLarochelleH2008-small}, and
sparse coding variants~\citep{HonglakLee-2007,ranzato-07-small}.
However, the most impressive recent results have been obtained with purely
supervised learning techniques for deep networks, in particular for speech
recognition~{\citep{dahl2010phonerec-small,Deng-2010,Seide2011}
and object recognition~\citep{Krizhevsky-2012-small}.
The latest breakthrough in object
recognition~\citep{Krizhevsky-2012-small} was achieved with fairly deep
convolutional networks with a form of noise injection in the input and hidden
layers during training,
called dropout~\citep{Hinton-et-al-arxiv2012}.

\ifLongversion
In all of these cases, the availability of large quantities of labeled data was critical.

On the other hand, progress with deep unsupervised architectures has been
slower, with the established approaches with a probabilistic footing
being the Deep Belief Network (DBN)~\citep{Hinton06} and the Deep Boltzmann
Machine (DBM)~\citep{Salakhutdinov+Hinton-2009-small}. 
Although single-layer unsupervised learners are fairly well
developed and used to pre-train these deep models, 
jointly training all the layers with respect to a single
unsupervised criterion remains a challenge, with a few
techniques arising to reduce that difficulty~\citep{Montavon2012,Goodfellow-et-al-NIPS2013}.
In contrast to recent progress toward joint supervised training of models with many layers, joint 
unsupervised training of deep models remains a difficult task.
\fi

\ifLongversion
In particular, the normalization constant
involved in complex multimodal probabilistic models is often intractable
and this is dealt with using various approximations (discussed below)
whose limitations may be an important part of the difficulty for training
and using deep unsupervised, semi-supervised or structured output models.
\fi

Though the goal of training large unsupervised networks has turned out to
be more elusive than its supervised counterpart, the vastly larger
available volume of unlabeled data still beckons for efficient methods to
model it.  Recent progress in training supervised models raises the question:
can we take advantage of this progress to improve our ability to train
deep, generative, unsupervised, semi-supervised or structured output models?

This paper lays theoretical foundations
for a move in this direction through the following main contributions:

{\bf  1 -- Intuition: } In Section~\ref{sec:unsup_hard} we discuss what we view as
basic motivation for studying alternate ways of training unsupervised probabilistic models, i.e., avoiding
the intractable sums or maximization involved in many approaches.

{\bf   2 -- Training Framework: } We start Section~\ref{sec:gsn} by presenting our recent work on the generative view of denoising auto-encoders (Section~\ref{sec:DAE-model-prob-density}). We present the \emph{walkback} algorithm which addresses some of the training difficulties with denoising auto-encoders (Section~\ref{sec:walkback-description}). 

We then generalize those results by introducing
latent variables in the framework to define Generative Stochastic Networks
(GSNs) (Section~\ref{sec:from-DAE-to-GSN}). GSNs aim to estimate the data-generating
distribution indirectly, by parametrizing the transition operator of a
Markov chain rather than directly parametrizing a model $P(X)$ of the observed
random variable $X$.
Most critically, {\em this framework transforms the unsupervised density estimation
problem into one which is more similar to supervised function approximation}.
This enables training by (possibly regularized) maximum likelihood
and gradient descent computed via simple back-propagation,
avoiding the need to compute intractable partition functions. Depending
on the model, this may allow us to draw from any number of recently demonstrated
supervised training tricks. 
\ifLongversion
For example, one could use a convolutional
architecture with max-pooling for parametric parsimony and computational
efficiency, 
or dropout ~\citep{Hinton-et-al-arxiv2012} to prevent co-adaptation of hidden
representations.
\fi

{\bf   3 -- General theory:}
Training the generative (decoding / denoising) component of a GSN 
$P(X|h)$ with noisy representation $h$ is often far easier than modeling $P(X)$ explicitly
(compare the blue and red distributions in \mbox{Figure~\ref{fig:data_px})}. We prove that if our estimated $P(X|h)$
is consistent (e.g. through maximum likelihood), then the
stationary distribution of the resulting Markov chain is a consistent estimator of the
data-generating density, ${\calP}(X)$ (Section~\ref{sec:DAE-model-prob-density} and
Appendix \ref{sec:local-consistency}).

{\bf   4 -- Consequences of theory: } We show that the model is general and extends to a wide range of architectures,
including sampling procedures whose computation can be unrolled
as a Markov Chain, i.e., architectures that add noise during intermediate
computation in order to produce random samples of a desired distribution (Theorem~\ref{thm:noisy-reconstruction}).
An exciting frontier in machine learning is the problem
of modeling so-called structured outputs, i.e., modeling
a conditional distribution where the output is high-dimensional
and has a complex multimodal joint distribution (given the input variable).
We show how GSNs can be used to support such structured output and missing values (Section~\ref{sec:missing_inputs}).

{\bf   5 -- Example application: } In Section~\ref{sec:gsn_experiment} we show an example application of the GSN theory to create a
deep GSN whose computational graph 
resembles the one followed by Gibbs sampling in deep Boltzmann machines (with continuous latent variables),
but that can be trained efficiently with back-propagated gradients
and without layerwise pretraining. 
Because the Markov Chain is defined over a state
$(X,h)$ that includes latent variables, we reap the dual advantage
of more powerful models for a given number of
parameters and better mixing in the chain as we add noise to variables
representing higher-level information, first suggested by the results
obtained by~\citet{Bengio-et-al-ICML2013}
and~\citet{Luo+al-AISTATS2013-small}. The experimental results show that
such a model with latent states indeed mixes better
than shallower models without them (Table~\ref{tab:LL}).

{\bf   6 -- Dependency networks: } Finally, an unexpected result falls out of the GSN theory:
it allows us to provide a novel justification for dependency networks~\citep{HeckermanD2000}
and for the first time define a proper joint distribution between all the visible
variables that is learned by such models (Section~\ref{sec:dependency-nets}).


\svs{2}
\section{Summing over too many major modes}
\label{sec:unsup_hard}
\svs{2}

The approach presented in this paper is motivated by a difficulty often encountered with 
probabilistic models, especially those containing anonymous latent variables.
They are called anonymous
because no a priori semantics are assigned to them, like in Boltzmann machines,
and unlike in many knowledge-based graphical models. Whereas inference
over non-anonymous latent variables is required to make sense of the model,
anonymous variables are only a device to capture the structure of the distribution
and need not have a clear human-readable meaning.

However, graphical models with latent variables often require dealing
with either or both of the following fundamentally difficult problems in
the inner loop of training, or to actually use the model for making
decisions: inference (estimating the posterior distribution over latent
variables $h$ given inputs $x$) and sampling (from the joint model of $h$
and $x$). However, if the posterior $P(h|x)$ has a huge number of modes
that matter, then the approximations made may break down.

Many of the computations involved in graphical models (inference, sampling,
and learning) are made intractable and difficult to approximate because
of the large number of non-negligible modes in the modeled distribution
(either directly $P(x)$ or a joint distribution $P(x,h)$ involving latent variables $h$).
In all of these cases, what is intractable is the computation or approximation
of a sum (often weighted by probabilities), such as a marginalization or the estimation
of the gradient of the normalization constant. If only a few terms in this
sum dominate (corresponding to the dominant modes of the distribution), then
many good approximate methods can be found, such as Monte-Carlo Markov
chains (MCMC) methods.

Deep Boltzmann machines~\citep{Salakhutdinov+Hinton-2009-small} combine the difficulty of
inference (for the {\em positive phase} where one tries to push the
energies associated with the observed $x$ down) and also that of
sampling (for the {\em negative phase} where one tries to push up the
energies associated with $x$'s sampled from $P(x)$).  
Sampling for the
negative phase is usually done by MCMC, although some unsupervised learning
algorithms~\citep{CollobertR2008-small,Gutmann+Hyvarinen-2010-small,Bordes-et-al-LSML2013}
involve ``negative examples'' that are sampled through simpler procedures
(like perturbations of the observed input, in a spirit reminiscent of
the approach presented here). 
Unfortunately, using an MCMC method to sample from $P(x,h)$ in order to
estimate the gradient of the partition function may be seriously hurt by
the presence of a large number of important modes, as argued below.

To evade the problem of highly multimodal joint or posterior
distributions, the currently known
approaches to dealing with the above intractable sums make very strong explicit 
assumptions (in the parametrization) or
implicit assumptions (by the choice of approximation methods) on the form of the distribution of interest.
In particular, MCMC methods are more likely to produce a good estimator
if the number of non-negligible modes is small: otherwise the
chains would require at least as many MCMC steps as the number of such
important modes, times a factor that accounts for the mixing time
between modes. Mixing time itself can be  very problematic as a trained model
becomes sharper, as it approaches a data-generating distribution
that may have well-separated and sharp modes (i.e., manifolds)~\citep{Bengio-et-al-ICML2013}.

We propose to make another assumption that might suffice to bypass
this multimodality problem: the
effectiveness of function approximation. 
As is typical in machine learning,
we postulate a rather large and flexible family of functions (such as deep
neural nets) and then use all manner of tricks to pick a member from that
combinatorially large family (i.e. to train the neural net) that both fits
observed data and generalizes to unseen data well.

In particular, the GSN approach presented in the next section relies on estimating the transition operator
of a Markov chain, e.g. $P(x_t | x_{t-1})$ or $P(x_t, h_t | x_{t-1}, h_{t-1})$.
Because each step of the Markov chain is generally local, these transition
distributions will often include only a very small number of important modes
(those in the neighborhood of the previous state). Hence the gradient of their
partition function will be easy to approximate. For example consider the
denoising transitions studied by~\citet{Bengio-et-al-NIPS2013} and illustrated
in Figure~\ref{fig:data_px},
where $\tilde{x}_{t-1}$ is a stochastically corrupted version of $x_{t-1}$
and we learn the denoising distribution $P(x | \tilde{x})$.
In the extreme case (studied empirically here) where $P(x | \tilde{x})$
is approximated by a unimodal distribution, the only form of training
that is required involves function approximation (predicting the clean $x$
from the corrupted $\tilde{x}$).

Although having the true $P(x | \tilde{x})$ turn out to be unimodal
makes it easier to find an appropriate family of models for it,
unimodality is by no means required by the GSN framework itself.  One
may construct a GSN using any multimodal model for output (e.g. mixture
of Gaussians, RBMs, NADE, etc.), provided that gradients for the parameters of
the model in question can be estimated (e.g. log-likelihood gradients).



The approach proposed here thus avoids the need for a poor approximation
of the gradient of the partition function in the inner loop of training,
but still has the potential of capturing very rich distributions by relying
mostly  on ``function approximation''.

Besides the approach discussed here, there may well be other
very different ways of evading this problem of 
intractable marginalization, including
approaches such as sum-product
networks~\citep{Poon+Domingos-2011b}, which are based on learning a probability function that
has a tractable form by construction and yet is from a flexible enough family
of distributions. Another interesting direction of investigation
that avoids the need for MCMC and intractable partition functions
is the variational auto-encoder~\citep{Kingma+Welling-ICLR2014,Gregor-et-al-ICML2014,Mnih+Gregor-ICML2014,Rezende-et-al-arxiv2014} and related directed models~\citep{Bornschein+Bengio-arxiv2014-small,Ozair+Bengio-arxiv2014},
which rely on learned approximate inference.

\svs{2}
\section{Generative Stochastic Networks}
\label{sec:gsn}
\svs{2}

In this section we work our way from denoising auto-encoders (DAE) to generative stochastic networks (GSN).
We illustrate the usefulness of denoising auto-encoders being applied iteratively as a way
to generate samples (and model a distribution). We introduce the {\em walkback} training algorithm
and show how it can facilitate the training.

We generalize the theory to GSNs, and provide a theorem that serves as a recipe 
as to how they can be trained. We also reference a classic result from matrix perturbation theory
to analyze the behavior of GSNs in terms of their stationary distribution.

We then study how GSNs may be used to fill missing values and theoretical
conditions for estimating associated conditional samples. 
Finally, we connect GSNs to dependency nets
and show how the GSN framework fixes one of the main problems with the theoretical
analysis of dependency nets and propose a particular way of sampling from them.

\svs{2}
\subsection{Denoising auto-encoders to model probability distributions}
\label{sec:DAE-model-prob-density}
\svs{2}

Assume the problem we face is to construct a model for some unknown
data-generating distribution ${\calP}(X)$ given only examples of $X$ drawn
from that distribution. In many cases, the unknown distribution
${\calP}(X)$ is complicated, and modeling it directly can be difficult.

A recently proposed approach using denoising auto-encoders (DAE) transforms the 
difficult task of modeling ${\calP}(X)$ into a supervised learning problem that may be much easier to solve.
The basic approach is as follows: given a clean example data point $X$ from ${\calP}(X)$,
we obtain a corrupted version $\tilde{X}$ by sampling from some corruption distribution ${\cal C}(\tilde{X}|X)$.
For example, we might take a clean image, $X$, and add random white noise to produce $\tilde{X}$.
We then use supervised learning methods to train a function to reconstruct, as accurately as possible,
any $X$ from the data set given only a noisy version $\tilde{X}$. As shown in Figure~\ref{fig:data_px},
the reconstruction distribution ${\calP}(X|\tilde{X})$ may often be much easier to learn than the data distribution ${\calP}(X)$,
{\em because ${\calP}(X|\tilde{X})$ tends to be dominated by a single or few major modes}
(such as the roughly Gaussian shaped density in the figure). What we call a major
mode is one that is surrounded by a substantial amount of probability mass. There
may be a large number of minor modes that can be safely ignored in the context
of approximating a distribution, but the major modes should not be missed.

But how does learning the reconstruction distribution help us solve our
original problem of modeling ${\calP}(X)$? The two problems are clearly
related, because if we knew everything about ${\calP}(X)$, then our
knowledge of the ${\cal C}(\tilde{X}|X)$ that we chose would allow us to
precisely specify the optimal reconstruction function via Bayes rule:
${\calP}(X|\tilde{X}) = \frac{1}{z} {\cal C}(\tilde{X}|X){\calP}(X)$,
where $z$ is a normalizing constant that does not depend on $X$. As one
might hope, the relation is also true in the opposite direction: once we
pick a method of adding noise, ${\cal C}(\tilde{X}|X)$, knowledge of the
corresponding reconstruction distribution ${\calP}(X|\tilde{X})$ is
sufficient to recover the density of the data ${\calP}(X)$.

In a recent paper, \citet{Alain+Bengio-ICLR2013} showed that denoising auto-encoders
with small Gaussian corruption and squared error loss estimated the
score (derivative of the log-density with respect to the input) of 
continuous observed random variables, thus implicitly estimating $P(X)$.
The following Proposition \ref{prop:markov-chain-basic-gsn}
generalizes this to arbitrary variables (discrete, continuous or both),
arbitrary corruption (not necessarily asymptotically small), and arbitrary
loss function (so long as they can be seen as a log-likelihood).

\begin{proposition}
\label{prop:markov-chain-basic-gsn}
Let $P(X)$ be the training distribution for which we only have empirical samples.
Let ${\cal C}(\tilde{X}|X)$ be the fixed corruption distribution and $P_\theta(X|\tilde{X})$
be the trained reconstruction distribution (assumed to have sufficient capacity).
We define a Markov chain that starts at some $X_0 \sim P(X)$ and then iteratively samples pairs of values
$(X_k, \tilde{X}_k)$ by alternatively sampling from 
${\cal C}(\tilde{X}_k|X_k)$ and from $P_\theta(X_{k+1}|\tilde{X}_k)$.

\begin{center}
\includegraphics[width=0.85\linewidth]{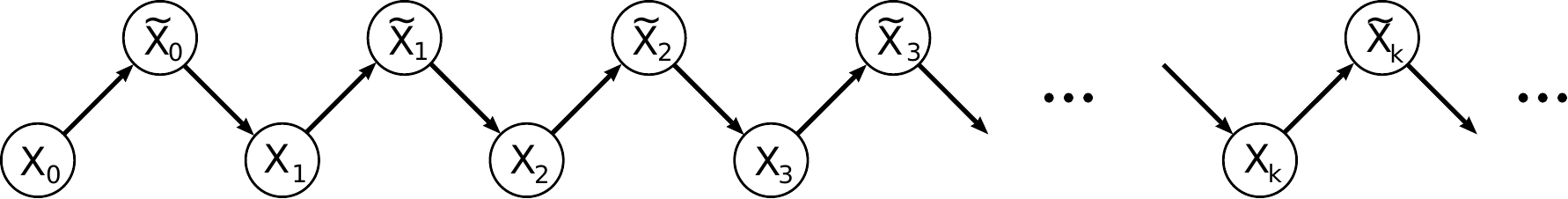}
\end{center}

Let $\pi$ be the stationary distribution of this Markov chain when we consider only the
sequence of values of $\{X_k\}_{k=0}^\infty$.

If we assume that this Markov chain is irreducible, that its stationary distribution exists,
and if we assume that $P_\theta(X|\tilde{X})$ is the distribution
that minimizes optimally the following expected loss
\begin{equation*}
\mathcal{L} = \int_{\tilde{X}} \int_X P(X)\mathcal{C}(\tilde{X}|X) \log P_\theta(X | \tilde{X}) dX d\tilde{X},
\end{equation*}
then we have that the stationary distribution $\pi$ is the same as the training distribution $P(X)$.
\end{proposition}
\begin{proof}
If we look at the density $P(\tilde{X})=\int P(X) \mathcal{C}(\tilde{X}|X) d\tilde{X}$ that
we get for $\tilde{X}$ by applying $\mathcal{C}(\tilde{X}|X)$ to the training data from $P(X)$,
we can rewrite the loss as a KL divergence
\begin{equation*}
\int_{\tilde{X}} \int_X P(X)\mathcal{C}(\tilde{X}|X) \log P_\theta(X | \tilde{X}) dX d\tilde{X} =
- \textrm{KL}\left( P(X)\mathcal{C}(\tilde{X}|X) \| P_\theta(X | \tilde{X}) P(\tilde{X}) \right) + \textrm{cst}
\end{equation*}
where the constant is independent of $P_\theta(X | \tilde{X})$.
This expression is maximized when
we have a $P_\theta(X | \tilde{X})$ that satisfies
\begin{equation}
\label{eqn:two-expressions-for-joint}
P(X)\mathcal{C}(\tilde{X}|X) = P_\theta(X | \tilde{X}) P(\tilde{X}).
\end{equation}
In that case, we have that
\begin{equation*}
P_{\theta^*}(X | \tilde{X}) = \frac{ P(X)\mathcal{C}(\tilde{X}|X) }{ P(\tilde{X} } = P(X | \tilde{X})
\end{equation*}
where $P(X | \tilde{X})$ represents the true conditional that we get through
the usual application of Bayes' rule.

Now, when we sample iteratively between $\mathcal{C}(\tilde{X}_k|X_k)$
and $P_{\theta^*}(X_{k+1} | \tilde{X}_k)$ to get the Markov chain illustrated above,
we are performing Gibbs sampling. We understand what Gibbs sampling does,
and here we are sampling using the two possible ways of expressing the joint from
equation (\ref{eqn:two-expressions-for-joint}). This means that the stationary
distribution $\pi$ of the Markov chain will have $P(X)$ as marginal density
when we look only at the $X_k$ component of the chain.

\end{proof}

Beyond proving that ${\calP}(X|\tilde{X})$ is sufficient to reconstruct the data density,
Proposition \ref{prop:markov-chain-basic-gsn} also demonstrates a method of sampling
from a learned, parametrized model of the density, $P_\theta(X)$, by running a Markov chain
that alternately adds noise using ${\cal C}(\tilde{X}|X)$ and denoises by sampling from
the learned $P_\theta(X|\tilde{X})$, which is trained to approximate the true ${\calP}(X|\tilde{X})$.

Before moving on, we should pause to make an important point clear.
Alert readers may have noticed that ${\calP}(X|\tilde{X})$ and ${\calP}(X)$
can each be used to reconstruct the other given knowledge of ${\cal C}(\tilde{X}|X)$.
Further, if we assume that we have chosen a simple ${\cal C}(\tilde{X}|X)$
(say, a uniform Gaussian with a single width parameter),
then ${\calP}(X|\tilde{X})$ and ${\calP}(X)$ must both be of approximately the same complexity.
Put another way, we can never hope to combine a simple ${\cal C}(\tilde{X}|X)$ and
a simple ${\calP}(X|\tilde{X})$ to model a complex ${\calP}(X)$.
Nonetheless, it may still be the case that ${\calP}(X|\tilde{X})$ is easier to {\em model} than ${\calP}(X)$
due to reduced computational complexity in computing or approximating the
partition functions of the conditional distribution mapping corrupted input $\tilde{X}$ to 
the distribution of corresponding clean input $X$.
Indeed, because that conditional is going to be mostly assigning probability to $X$ locally around $\tilde{X}$,
${\calP}(X|\tilde{X})$ has only one or a few major modes, while ${\calP}(X)$ can have a very large number
of them.

So where did the complexity go? ${\calP}(X|\tilde{X})$ has fewer major modes than ${\calP}(X)$,
but {\em the location of these modes depends on the value of $\tilde{X}$}.
It is precisely this mapping from $\tilde{X} \rightarrow$ {\em mode location} that allows us
to trade a difficult density modeling problem for a supervised function approximation problem
that admits application of many of the usual supervised learning tricks.


\vspace{0.5cm}

In the Gaussian noise example, what happens is that the tails of the Gaussian
are exponentially damping all but the modes that are near $X$, thus preserving
the actual number of modes but considerably changing the number
of major modes. 
In the Appendix we also present one alternative line of reasoning based
on a corruption process $C(\tilde{X}|X)$ that has finite local support,
thus completely removing the modes that are not in the neighborhood of $X$.
We argue that even with such a corruption process,
the stationary distribution $\pi$ will match the original $P(X)$,
so long as one can still visit all the regions of interest through
a sequence of such local jumps.

Two potential issues with Proposition \ref{prop:markov-chain-basic-gsn}
are that 1) we are learning distribution $P_\theta(X | \tilde{X})$ based on
experimental samples so it is only asymptotically minimizing the desired loss,
and 2) we may not have enough capacity in our model to estimate $P_{\theta}(X | \tilde{X})$
perfectly.

The issue is that, when running a Markov chain for infinitely long using a
slightly imperfect $P_\theta(X | \tilde{X})$, these small differences may
affect the stationary distribution $\pi$ and compound over time. We are
not allowed to ``adjust'' the $P_\theta(X | \tilde{X})$ as the chain runs.

This is addressed by Theorem \ref{thm:schweitzer_inequality} cited in
the later Section \ref{sec:from-DAE-to-GSN}. That theorem gives us a result about continuity,
so that, for ``well-behaved'' cases, when $P_\theta(X | \tilde{X})$ is close
to $P(X | \tilde{X})$ we must have that the resulting stationary distribution
$\pi$ is close to the
original $P(X)$.



\svs{2}
\subsection{Walkback algorithm for training denoising auto-encoders}
\label{sec:walkback-description}
\svs{2}

In this section we describe the walkback algorithm which is
very similar to the method from Proposition 1, but helps training
to converge faster. It differs in the
training samples that are used, and the fact that the solution
is obtained through an iterative process. The parameter update
changes the corruption function, which changes the $\tilde{X}$ in
the training samples, which influences the next parameter update, and so on.

Sampling in high-dimensional spaces (like in experiments
in Section \ref{sec:walkback_experiment})
using a simple local corruption process (such as Gaussian or
salt-and-pepper noise) suggests that
if the corruption is too local, the DAE's behavior far
from the training examples can create spurious modes in the regions
insufficiently visited during training. More training iterations or 
increasing the amount of corruption noise
helps to substantially alleviate that problem, but we discovered an even
bigger boost by {\em training
the Markov chain to walk back towards the training examples}
(see Figure \ref{fig:walkback_into_drain}).
\begin{SCfigure}
\centering
\includegraphics[width=0.49\textwidth]{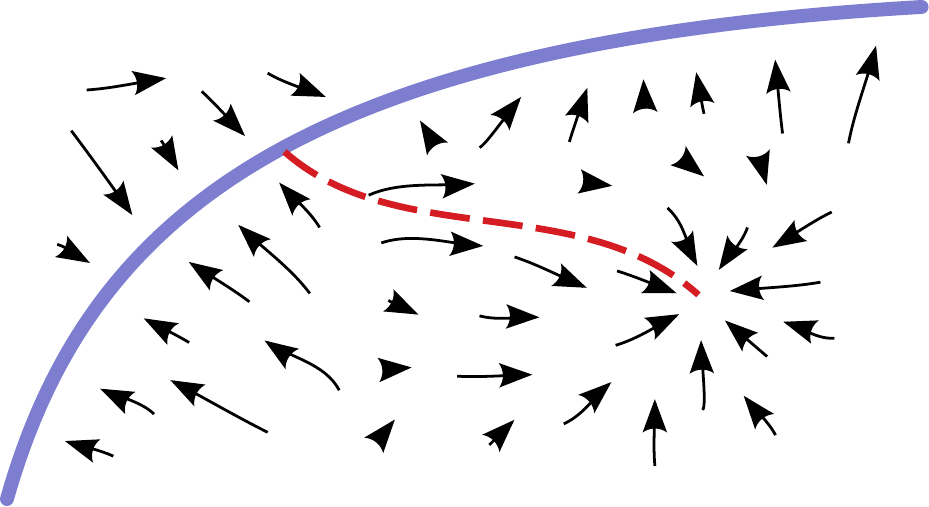}
\caption{
Walkback samples get attracted by spurious modes and contribute to removing them.
Segment of data manifold in violet and example walkback path in red dotted line,
starting on the manifold and going towards a spurious attractor.
The vector field represents expected moves of the chain,
for a unimodal $P(X| \tilde{X})$, with arrows from $\tilde{X}$ to $X$.
The name {\bf walkback} is because this procedure forces the model
to learn to {\em walk back from the random walk it generates,
towards the $X$'s in the training set}.
}
\label{fig:walkback_into_drain}
\vspace{-0.5cm}
\end{SCfigure}
We exploit
knowledge of the currently learned model $P_\theta(X|\tilde{X})$ to define the
corruption, so as to pick values of $\tilde{X}$ that would be
obtained by following the generative chain: wherever
the model would go if we sampled using the generative Markov chain starting
at a training example $X$,
we consider to be a kind of ``negative example'' $\tilde{X}$
from which the
auto-encoder should move away (and towards $X$).
The spirit of this procedure is thus very similar
to the CD-$k$ (Contrastive Divergence with $k$ MCMC steps) procedure
proposed to train RBMs~\citep{Hinton99-small,Hinton06}.

\vspace{1cm}

We start by defining the modified corruption process $\mathcal{C}_k(\tilde{X}|X)$
that samples $k$ times alternating between $\mathcal{C}(\tilde{X}|X)$ and the
current $P_\theta(X | \tilde{X})$.

We can express this recursively if we let $\mathcal{C}_1(\tilde{X}|X)$ be
our original $\mathcal{C}(\tilde{X}|X)$, and then define
\begin{equation}
\mathcal{C}_{k+1}(\tilde{X}|X) = \int_{\tilde{X'}} \int_{X'} \mathcal{C}(\tilde{X} | X') P_\theta( X' | \tilde{X}') \mathcal{C}_k(\tilde{X}'|X) dX' d\tilde{X}'
\end{equation} 
Note that this corruption distribution $\mathcal{C}_k(\tilde{X}|X)$ now involves
the distribution $P_\theta(X|\tilde{X})$ that we are learning.

With the help of the above definition of $\mathcal{C}_k(\tilde{X}|X)$, we define
the walkback corruption process $\mathcal{C}_{\textrm{\textsc{wb}}}(\tilde{X}|X)$.
To sample from $\mathcal{C}_{\textrm{\textsc{wb}}}$, we first draw a $k$
distributed according to some distribution, e.g., a geometric distribution 
with parameter $p=0.5$ and support on $k\in\{1,2,\ldots\}$),
and then we sample according to the corresponding $\mathcal{C}_k(\tilde{X}|X)$.
Other values than $p=0.5$ could be used, but we just want something convenient
for that hyperparameter.
Conceptually, the corruption process
$\mathcal{C}_{\textrm{\textsc{wb}}}$ means that, from a starting point $X$
we apply iteratively the original $\mathcal{C}$ and $P_\theta$, and then we flip
a coin to determine if we want to do it again. We re-apply until we lose the coin flip,
and then this gives us a final value for the sample $\tilde{X}$ based on $X$.

The walkback loss is given by
\begin{equation}
\label{eqn:loss_walkback_empirical}
\mathcal{L}_{\textrm{\textsc{wb}}} \simeq \frac{1}{N} \sum_{i=1}^N \log P_\theta(X^{(i)} | \tilde{X}^{(i)})
\end{equation}
for samples ${(X^{(i)}, k^{(i)}, \tilde{X}^{(i)})}$ drawn from $X \sim P(X)$, $k \sim \textrm{Geometric(0.5)}$ and
$\tilde{X} \sim \mathcal{C}_k(\tilde{X}|X)$.
Minimizing this loss is an iterative process because the samples used in the empirical expression
depend on the parameter $\theta$ to be learned. This iterated minimization is what we
call the \emph{walkback algorithm}. Samples are generated with the current parameter value $\theta_t$,
and then the parameters are modified to reduce the loss and yield $\theta_{t+1}$. We repeat until
the process stabilizes. In practical applications, we do not have infinite-capacity models
and we do not have a guarantee that the walkback algorithm should converge to some $\theta^*$.

\subsubsection{Reparametrization Trick}
\label{sec:reparametrization}

Note that we do not need to analytically marginalize over the latent variables involved:
we can back-propagate through the chain, considering it like a recurrent neural network
with noise (the corruption) injected in it. This is an instance of 
the so-called reparametrization trick, already proposed in~\citep{Bengio-arxiv2013,Kingma-arxiv2013,Kingma+Welling-ICLR2014}.
The idea is that we can consider sampling from a random variable conditionally on others (such as $\tilde{X}$
given $X$) as equivalent to applying a deterministic function taking as argument the conditioning variables as
well as some i.i.d. noise sources.
This view is particularly useful for the more general
GSNs introduced later, in which we typically choose the latent variables to be continuous, i.e.,
allowing to backprop through their sampling steps when exploiting the reparametrization trick.

\subsubsection{Equivalence of the Walkback Procedure}

With the walkback algorithm, one can also decide to include or not in the loss function
all the intermediate reconstruction distributions through which the trajectories pass. That is,
starting from some $X_0$, we sample
\begin{eqnarray*}
X_0 \sim P(X)                                     & \tilde{X_0} \sim \mathcal{C}(\tilde{X_0} | X_0), \\
X_1 \sim P_\theta( X_1 | \tilde{X}_0)             & \tilde{X_1} \sim \mathcal{C}(\tilde{X_1} | X_1) \\
X_2 \sim P_\theta( X_2 | \tilde{X}_1)             & \tilde{X_2} \sim \mathcal{C}(\tilde{X_2} | X_2) \\
\vdots                                            & \vdots \\
X_{k-1} \sim P_\theta( X_{k-1} | \tilde{X}_{k-2}) & \tilde{X}_{k-1} \sim \mathcal{C}(\tilde{X}_{k-1} | X_{k-1})
\end{eqnarray*}
and we use all the pairs $(X, \tilde{X_k})$ as training data for the
walkback loss at equation (\ref{eqn:loss_walkback_empirical}).

The following proposition looks very similar to Proposition \ref{prop:markov-chain-basic-gsn},
but it uses the walkback corruption instead of the original corruption ${\cal C}(\tilde{X}|X)$.
It is also an iterated process through which the current value of the parameter $\theta_t$
sets the loss function that will be minimized by the updated $\theta_{t+1}$.

\begin{proposition}
\label{prop:walkback}
Let $P(X)$ be the training distribution for which we only have empirical samples.
Let $\pi(X)$ be the implicitly defined asymptotic distribution of the Markov chain alternating
sampling from $P_\theta(X|\tilde{X})$ and ${\cal C}(\tilde{X}|X)$, where ${\cal C}$
is the original local corruption process.

If we assume that $P_\theta(X|\tilde{X})$ has sufficient capacity and
that the walkback algorithm converges (in terms of being stable in the updates
to $P_\theta(X|\tilde{X})$), then $\pi(x)=P(X)$.

That is, the Markov chain defined by alternating $P_\theta(X|\tilde{X})$ and ${\cal C}(\tilde{X}|X)$
gives us samples that are drawn from the same distribution as the training data.
\vs{1}
\end{proposition}
\vs{3}
\begin{proof}

Consider that during training, we produce a sequence of estimators
$P_{\theta_t}(X|\tilde{X})$ where $P_{\theta_t}$ corresponds to the $t$-th training iteration
(modifying the parameters after each iteration). With the walkback
algorithm, $P_{\theta_{t-1}}$ is used to obtain the corrupted samples $\tilde{X}$
from which the next model $P_{\theta_{t-1}}$ is produced.

If training converges in terms of ${\theta_t} \rightarrow \theta^*$, it means that
we have found a value of $P_{\theta^*}(X|\tilde{X})$ such that
\begin{equation*}
\theta^* = \textrm{argmin}_{\theta} \frac{1}{N} \sum_{i=1}^N \log P_\theta(X^{(i)} | \tilde{X}^{(i)})
\end{equation*}
for samples ${(X^{(i)}, \tilde{X}^{(i)})}$ drawn from $X \sim P(X)$,
$\tilde{X} \sim \mathcal{C}_{\textrm{\textsc{wb}}}(\tilde{X}|X)$.

By Proposition \ref{prop:markov-chain-basic-gsn}, we know that, regardless of the the corruption
$\mathcal{C}_{\textrm{\textsc{any}}}(\tilde{X}|X)$ used,
when we have a $P_{\theta}(X|\tilde{X})$ that minimizes optimally the loss
\begin{equation*}
\int_{\tilde{X}} \int_X P(X)\mathcal{C}_{\textrm{\textsc{any}}}(\tilde{X}|X) \log P_\theta(X | \tilde{X}) dX d\tilde{X}
\end{equation*}
then we can recover $P(X)$ by alternating between
$\mathcal{C}_{\textrm{\textsc{any}}}(\tilde{X}|X)$ and $P_\theta(X | \tilde{X})$.

Therefore, once the model is trained with walkback, the stationary distribution $\pi$
of the Markov chain that it creates has the same distribution $P(X)$ as the training data.


Hence if we alternate between the original corruption $\mathcal{C}(\tilde{X}|X)$
and the walkback solution $P_{\theta^*}( X | \tilde{X} )$, then the stationary distribution
with respect to $X$ is also $P(X)$.
\vs{3}
\end{proof}

Note that this proposition applies regardless of the value of geometric distribution used
to determine how many steps of corruption will be used. It applies whether we keep
all the samples along to the way, or only the one at the last step. It applies regardless
of if we use a geometric distribution to determine which $\mathcal{C}_k$ to select, or any
other type of distribution.

A consequence is that {\em the walkback training
algorithm estimates the same distribution as the original denoising algorithm}, but
may do it more efficiently (as we observe in the experiments),
by exploring the space of corruptions in a way that spends more time
where it most helps the model to kill off spurious modes.

The Markov chain that we get with walkback should also generally mix faster, be less
susceptible to getting stuck in bad modes, but it will require a $P_{\theta^*}( X | \tilde{X} )$
with more capacity than originally. This is because $P_{\theta^*}( X | \tilde{X} )$
is now less local, covering the values of the initial $X$ that could have given rise to the $\tilde{X}$
resulting from several steps of the Markov chain.

\svs{2}
\subsection{Walkbacks with individual scaling factors to handle uncertainty} \label{sec:swb}
The use of the proposed walkback training procedure is effective in suppressing the spurious modes 
in the learned data distribution. Although the convergence is guaranteed asymptotically, in practice, 
given limited model capacity and training data, it has been observed that 
the more walkbacks in training, the more difficult it is to 
maximize $P_{\theta}( X | \tilde{X} )$. 
This is simply because more and more noise is added in this procedure, 
resulting in $\tilde{X}$ that 
is further away from $X$, therefore a potentially more complicated reconstruction 
distribution. 

In other words, $P_{\theta}( X | \tilde{X} )$ needs to have the capacity to 
model increasingly 
complex reconstruction distributions. 
As a result of training, a simple, or usually unimodal $P_{\theta}( X | \tilde{X} )$ 
is most likely to 
learn a distribution with a larger 
uncertainty than the one learned without walkbacks in order to distribute some 
probability mass to the 
more complicated and multimodal distributions implied by the walkback training procedure. 
One possible 
solution to this problem is to use a multimodal reconstruction distribution such as in 
\citet{Sherjil-et-al-arxiv2014}, \cite{Larochelle+Murray-2011-small}, 
or~\citet{Dinh-et-al-arxiv2014}. We propose here another solution, which can be combined
with the above, that consists in allowing a different level of entropy for
different steps of the walkback.

\subsubsection{Scaling trick in binary $X$}\label{sec:swb_binary}
In the case of binary $X$, the most common choice of the 
reconstruction distribution is the factorized Multinoulli distribution where 
$P_{\theta}( X | \tilde{X} )=\prod_{i=1}^d P_{\theta}(X^i | \tilde{X})$ 
and $d$ is the dimensionality of $X$.
Each factor $P_{\theta}(X^i | \tilde{X})$ is modeled by a Bernoulli distribution 
that has its parameter $p_i=\sigmoid(f_i(\tilde{X}))$ where $f_i(\cdot)$ is a general nonlinear 
transformation realized by a neural network. We propose to use a different scaling factor 
$\alpha_k$ for different walkback steps, resulting in a new parameterization 
$p_i^k=\sigmoid(\alpha_k f_i(\tilde{X}))$ for the k-th walkback step, with $\alpha_k>0$
being learned.
$\alpha_k$ effectively scales the pre-activation of the sigmoid function according to 
the uncertainty or entropy associated with different walkback steps. 
Naturally, later reconstructions in the walkback sequence are less accurate because
more noise has been injected. Hence, given the $k_i$-th and $k_j$-th walkback steps 
that satisfy  $k_i < k_j$, the learning will tend to
result in $\alpha_{k_i} > \alpha_{k_j}$ because larger $\alpha_k$ correspond
to less entropy.

\subsubsection{Scaling trick in real-valued $X$}\label{sec:swb_real}
In the case of real-valued $X$, the most common choice of $P_{\theta}( X | \tilde{X} )$ 
is the factorized Gaussian. In particular, each factor $P_{\theta}(X^i | \tilde{X})$ is 
modeled by a Normal distribution with its parameters $\mu_i$ and $\sigma_i$. Using the 
same idea of learning separate scaling factors, we can parametrize it as 
$P_{\theta}(X^i | \tilde{X}) = \mathcal{N}(\mu_i, \alpha_k \sigma_i^2)$ for the $k$-th 
walkback step. $\alpha_k$ is positive and also learned. However, Given the $k_i$-th and 
$k_j$-th walkback steps 
that satisfy  $k_i < k_j$, the learning will 
result $\alpha_{k_i} < \alpha_{k_j}$, since in this case, larger $\alpha_k$ 
indicates larger entropy.

\subsubsection{Sampling with the learned scaling factors}
After learning the scaling factors $\alpha_k$ for $k$ different walkback steps, the 
sampling is straightforward. One noticeable difference is that we have learned $k$ 
Markov transition operators. Although, asymptotically all $k$ Markov chains 
generate the same distribution of $X$, in practice, they result in different distributions 
because of the different $\alpha_k$ learned. In fact, using $\alpha_1$ results the 
samples that are sharper and more faithful to the data distribution. We verify the 
effect of learning the scaling factor further in the experimental section.
    
\subsection{Extending the denoising auto-encoder to more general GSNs}
\label{sec:from-DAE-to-GSN}
\svs{2}


The denoising auto-encoder Markov chain is defined by 
$\tilde{X}_t \sim C(\tilde{X}|X_t)$ and 
$X_{t+1} \sim P_\theta(X | \tilde{X}_t)$, where $X_t$ alone can
serve as the state of the chain. The GSN framework generalizes the DAE in
two ways:
\begin{enumerate}
\item the ``corruption'' function is not fixed anymore but a parametrized function
that can be learned and corresponds to a ``hidden'' state 
(so we write the output of this function $H$ rather than $\tilde{X}$); and
\item that intermediate variable $H$ is now considered part of the state of the Markov chain,
i.e., its value of $H_t$ at step $t$ of the chain
depends not just on the previous visible $X_{t-1}$ but also 
on the previous state $H_{t-1}$.
\end{enumerate}
For this purpose, we define the
Markov chain associated with a GSN in terms of a visible $X_t$ and a 
latent variable $H_t$ as state
variables, of the form
\begin{eqnarray*}
   H_{t+1} &\sim& P_{\theta_1}(H | H_t, X_t) \\
   X_{t+1} &\sim& P_{\theta_2}(X | H_{t+1}).
\end{eqnarray*}
\begin{center}
\includegraphics[width=0.7\textwidth]{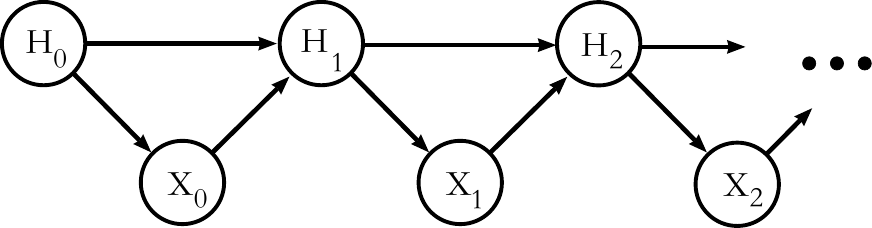} 
\end{center}
This definition makes denoising auto-encoders a special case of GSNs.
Note that, given that the distribution of $H_{t+1}$ may depend on a previous value of $H_t$,
we find ourselves with an extra $H_0$ variable added at the beginning of the chain.
This $H_0$ complicates things when it comes to training, but when
we are in a sampling regime we can simply wait a sufficient
number of steps to burn in.

\subsubsection{Main result about GSNs}

The next theoretical results give conditions for making the stationary distributions
of the above Markov chain match a target data-generating distribution. It basically
says that, in order to estimate the data-generating distribution $P(X_0)$, it is
enough to achieve two conditions.

The first condition is similar to the one we
obtain when minimizing denoising reconstruction error, i.e., we must
make sure that the reconstruction distribution $P(X_1 | H_1)$ approaches
the conditional distribution $P(X_0 | H_1)$, i.e., the $X_0$'s that could
have given rise to $H_1$.

The second condition is novel and regards the
initial state $H_0$ of the chain, which influences $H_1$. It says that
$P(H_0|X_0)$ must match $P(H_1|X_0)$. One way to achieve that is to initialize
$H_0$ associated with a training example $X_0$ with the previous value of $H_1$
that was sampled when example $X_0$ was processed. In the graphical model
in the statement of Theorem \ref{thm:noisy-reconstruction},
note how the arc relating $X_0$ and $H_0$ goes in the $X_0 \rightarrow H_0$ direction,
which is different from the way we would sample from the GSN (graphical
model above), where we have $H_0 \rightarrow X_0$. Indeed, during training,
$X_0$ is given, forcing it to have the data-generating distribution.

Note that Theorem \ref{thm:noisy-reconstruction} is there to provide us
with a guarantee about what happens when those two conditions are satisfied.
It is not originally meant to describe a training method.

In section \ref{sec:how-to-train-theorem3} we explain how to these conditions
could be approximately achieved.

\hspace{1em}

\pagebreak[2]

\begin{theorem}
\label{thm:noisy-reconstruction}
Let ${(H_t,X_t)}_{t=0}^\infty$ be the Markov chain defined by the following graphical model.
\begin{center}
\includegraphics[width=0.7\textwidth]{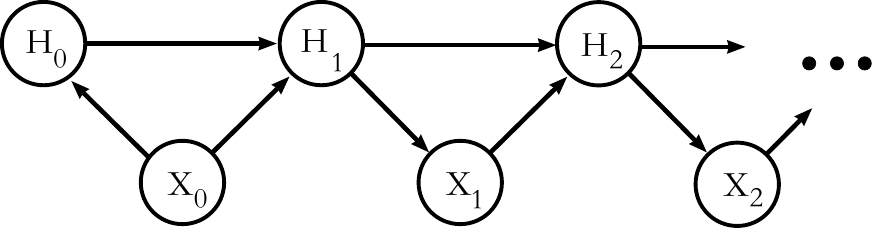} 
\end{center}
If we assume that the chain has a stationary distribution $\pi_{H,X}$, and that for every value of $(x,h)$ we have that
\begin{itemize}
\item all the $P(X_t = x | H_t=h) = g(x|h)$ share the same density for $t \geq 1$
\item all the $P(H_{t+1} = h | H_t=h', X_t=x) = f(h| h', x)$ shared the same density for $t \geq 0$
\item $P(H_0=h|X_0=x)=P(H_1=h|X_0=x)$
\item $P(X_1=x|H_1=h)=P(X_0=x|H_1=h)$
\end{itemize}
then for every value of $(x,h)$ we get that
\begin{itemize}
\item $P(X_0 = x | H_0=h) = g(x|h)$ holds, which is something that was assumed only for $t\geq 1$
\item $P(X_t = x, H_t = h) = P(X_0 = x, H_0 = h)$ for all $t\geq 0$
\item the stationary distribution $\pi_{H,X}$ has a marginal distribution $\pi_X$ such that $\pi\left(x\right) = P\left(X_0=x\right)$.
\end{itemize}
Those conclusions show that our Markov chain has the property that its samples in $X$ are drawn from the same distribution as $X_0$.
\end{theorem}

\begin{proof}
The proof hinges on a few manipulations done with the first variables to show
that $P(X_t = x | H_t=h) = g(x|h)$, which is assumed for $t \geq 1$, also holds for $t=0$.

For all $h$ we have that
\begin{eqnarray*}
P(H_0=h) & = & \int P(H_0=h|X_0=x)P(X_0=x) dx \\
         & = & \int P(H_1=h|X_0=x)P(X_0=x) dx \hspace{1em} \textrm{(by hypothesis)} \\
         & = & P(H_1=h).
\end{eqnarray*}

The equality in distribution between $(X_1,H_1)$ and $(X_0,H_0)$ is obtained with
\begin{eqnarray*}
P(X_1=x, H_1=h) & = & P(X_1=x | H_1=h) P(H_1=h) \\
                & = & P(X_0=x | H_1=h) P(H_1=h) \hspace{1em} \textrm{(by hypothesis)} \\
                & = & P(X_0=x, H_1=h) \\
                & = & P(H_1=h | X_0=x)P(X_0=x) \\
                & = & P(H_0=h | X_0=x)P(X_0=x) \hspace{1em} \textrm{(by hypothesis)} \\
                & = & P(X_0=x, H_0=h).
\end{eqnarray*}

Then we can use this to conclude that
\begin{eqnarray*}
         & P(X_0 = x, H_0 = h) = P(X_1 = x, H_1 = h) \\
\implies & P(X_0 = x |H_0 = h) = P(X_1 = x | H_1 = h) = g(x|h)
\end{eqnarray*}
so, despite the arrow in the graphical model being turned the other way, we have that
the density of $P(X_0 = x |H_0 = h)$ is the same as for all other $P(X_t = x |H_t = h)$ with $t \geq 1$.

Now, since the distribution of $H_1$ is the same as the distribution of $H_0$, and
the transition probability $P(H_1=h|H_0=h')$ is entirely defined by the $(f,g)$ densities
which are found at every step for all $t\geq 0$, then we know that $(X_2, H_2)$ will have the
same distribution as $(X_1, H_1)$. To make this point more explicitly,
\begin{eqnarray*}
 P(H_1 = h | H_0 = h') & = & \int P(H_1 = h | H_0 = h', X_0 = x)  P(X_0 = x | H_0 = h') dx \\
                       & = & \int f(h | h', x)  g(x| h') dx \\
                       & = & \int P(H_2 = h | H_1 = h', X_1 = x)  P(X_1 = x | H_1 = h') dx \\
                       & = & P(H_2 = h | H_1 = h')
\end{eqnarray*}
This also holds for $P(H_3|H_2)$ and for all subsequent $P(H_{t+1}|H_t)$.
This relies on the crucial step where we demonstrate that $P(X_0 = x | H_0 = h) = g(x|h)$.
Once this was shown, then we know that we are using the
same transitions expressed in terms of $(f,g)$ at every step.

Since the distribution of $H_0$ was shown above to be the same as the distribution of $H_1$,
this forms a recursive argument that shows that all the $H_t$ are equal in distribution to $H_0$.
Because $g(x|h)$ describes every $P(X_t = x | H_t = h)$, we have that
all the joints $(X_t, H_t)$ are equal in distribution to $(X_0, H_0)$.

This implies that the stationary distribution $\pi_{X,H}$ is the same as that of $(X_0, H_0)$.
Their marginals with respect to $X$ are thus the same.
\end{proof}

Intuitively, the proof of Theorem \ref{thm:noisy-reconstruction} achieves its objective by
forcing all the $(H_t, X_t)$ pairs to share the same joint distribution, thus making
the marginal over $X_t$ as $t \rightarrow \infty$ (i.e. the stationary distribution of the chain $\pi$) 
be the same as $P(X_0)$, i.e., the data distribution. On the other hand, because it is
a Markov chain, its stationary distribution does not depend on the initial conditions,
making the model generate from an estimator of $P(X_0)$ for any initial condition.

To apply Theorem \ref{thm:noisy-reconstruction} in a context
where we use experimental data to learn a model, we would like
to have certain guarantees concerning the robustness of
the stationary density $\pi_X$. When a model lacks capacity, or
when it has seen only a finite number of training examples,
that model can be viewed as a perturbed version of the exact
quantities found in the statement of
Theorem \ref{thm:noisy-reconstruction}.

\subsubsection{A note about consistency}
\label{sec:a-note-about-consistency-from-Schweitzer}

A good overview of results from perturbation theory
discussing stationary distributions in finite state Markov chains
can be found in \citep{Cho2000comparisonperturbation}.
We reference here only one of those results.

\begin{theorem}
\label{thm:schweitzer_inequality}
Adapted from \citep{Schweitzer1968perturbation}

Let $K$ be the transition matrix of a finite state, irreducible,
homogeneous Markov chain. Let $\pi$ be its stationary distribution
vector so that $K\pi=\pi$. Let $A=I-K$ and \mbox{$Z=\left(A+C\right)^{-1}$}
where $C$ is the square matrix whose columns all contain $\pi$.
Then, if $\tilde{K}$ is any transition matrix (that also satisfies
the irreducible and homogeneous conditions) with stationary distribution
$\tilde{\pi}$, we have that
\[
\left\Vert \pi-\tilde{\pi}\right\Vert _{1}\leq\left\Vert Z\right\Vert _{\infty}\left\Vert K-\tilde{K}\right\Vert _{\infty}.
\label{eqn:schweitzer_inequality}
\]

\end{theorem}

This theorem covers the case of discrete data by showing how the stationary
distribution is not disturbed by a great amount when the transition
probabilities that we learn are close to their correct values. We are
talking here about the transition between steps of the chain
$(X_0, H_0), (X_1, H_1), \ldots, (X_t, H_t)$, which are defined
in Theorem \ref{thm:noisy-reconstruction} through the $(f,g)$ densities.

\subsubsection{Training criterion for GSNs}
\label{sec:how-to-train-theorem3}

So far we avoided discussing the training criterion for a GSN.
Various alternatives exist, but this analysis is for future work.
Right now Theorem \ref{thm:noisy-reconstruction} suggests the following
rules :
\begin{itemize}
\item Define $g(x|h) = P(X_1=x|H_1=h)$, i.e., the {\em decoder}, to be the estimator for \mbox{$P(X_0=x|H_1=h)$},
e.g. by training an estimator of this conditional distribution from the samples $(X_0, H_1)$,
with reconstruction likelihood, $\log P(X_1 = x_0 | H_1)$, as this
would asymptotically achieve the condition  $P(X_0 | H_1) = P(X_1 | H_1)$. To see that this 
is true, consider the following.

We sample $X_0$ from $P(X_0)$ (the data-generating distribution) and $H_1$ from $P(H_1|H_0,X_0)$. 
Refer to one of the next bullet points for an explanation
about how to get values for $H_0$ to be used when sampling
from $P(H_1 | H_0, X_0)$ here.
This creates a joint distribution over $(X_0, H_1)$
that has $P(X_0 | H_1)$ as a derived conditional. 
Then we train the parameters of a model $P_\theta(X_1|H_1)$ to maximize the log-likelihood
\begin{align}
    & \mathbb{E}_{x_0\sim P(X_0), h_1\sim P(H_1|x_0)} [  \log P_\theta(X_1=x_0 | h_1) ]  \nonumber \\
   =& \int_{x_0, h_1} P(x_0,h_1) \log P_\theta(X_1=x_0 | H_1=h_1) dx_0 dh_1 \nonumber\\
   =& \int_{h_1} P(h_1) \int_{x_0} P(X_0=x_0 | H_1 = h_1) \log P_\theta(X_1=x_0 | h_1) dx_0 dh_1 \nonumber\\
   =& -\mathbb{E}_{H_1}[ \textrm{KL}(P(X_0|H_1) || P_\theta(X_1|H_1)) ] + {\rm const.}
\end{align}  
where the constant does not depend on $\theta$, and thus the log-likelihood is maximized when 
\[
  P_\theta(X_1=x|H_1=h) = P(X_0=x|H_1=h).
\]

\item Pick the transition distribution $f(h | h', x)$ to be useful, i.e., training it towards the
same objective, i.e., sampling an $h'$ that makes it easy to reconstruct $x$. One can think 
of $f(h | h', x)$ as the {\em encoder}, except that it has a state which depends on its previous
value in the chain.
\item To approach the condition $P(H_0 | X_0) = P(H_1 | X_0)$, one interesting possibility is the following.
For each $X_0$ in the training set, iteratively sample $H_1 | (H_0, X_0)$ and 
substitute the value of $H_1$ as the updated value of $H_0$. Repeat until you have achieved a kind of ``burn in''.
Note that, after the training is completed, when we use the chain for sampling, the samples that we get from 
its stationary distribution do not depend on $H_0$. Another option is to store the value of $H_1$ that was
sampled for the particular training example $x_0$, and re-use it as the initial $H_0$ the next time that $x_0$
is presented during training. These techniques of substituting $H_1$ into $H_0$ are only
required during training. In our experiments, we actually found that a fixed $H_0=0$
worked as well, so we have used this simpler approach in the reported experiments.
\item The rest of the chain for $t \geq 1$ is defined in terms of $(f,g)$.
\end{itemize}

\svs{2}
\subsection{Random variable as deterministic function of noise}
\label{sec:deterministic-function-of-noise}
\svs{2}


There several equivalent ways of
expressing a GSN. One of the interesting
formulations is to use deterministic functions
of random variables to express the densities $(f,g)$
used in Theorem \ref{thm:noisy-reconstruction}.
With that approach, we define $H_{t+1}=\otherf_{\theta_1}(X_t,Z_t,H_t)$
for some independent noise source $Z_t$,
and we insist that $X_t$ cannot be recovered
exactly from $H_{t+1}$, to avoid a situation in which the Markov
chain would not be ergodic. The advantage of that formulation
is that one can directly back-propagate the reconstruction
log-likelihood $\log P(X_1=x_0 | H_1=f(X_0,Z_0,H_0))$ into all the parameters
of $f$ and $g$, using the reparametrization trick discussed above
in Section~\ref{sec:reparametrization}.
This method is described in \mbox{\citep{Williams-1992}}.


In the setting described at the beginning of Section \ref{sec:gsn},
the function playing the role of the ``encoder'' was fixed for the purpose
of the theorem, and we showed that learning only the ``decoder'' part (but a
sufficiently expressive one) sufficed. In this
setting we are learning both, which can cause certain
broken behavior.

One problem would be if the created Markov
chain failed to converge to a stationary distribution.
Another such problem could be that the function $\otherf(X_t,Z_t,H_t)$
learned would try to ignore the noise $Z_t$, or not make the best use
out of it. In that case, the reconstruction distribution would
simply converge to a Dirac at the input $X$.  This is
the analogue of the constraint on auto-encoders that is needed to prevent
them from learning the identity function. Here, we must design the family
from which $f$ and $g$ are learned 
such that when the noise $Z$ is injected, there are always
several possible values of $X$ that could have been the correct original
input.

Another extreme case to think about is when $\otherf(X,Z,H)$ is overwhelmed
by the noise and has lost all information about $X$. In that case the theorems
are still applicable while giving uninteresting results: the learner must
capture the full distribution of $X$ in $P_{\theta_2}(X|H)$ because
the latter is now equivalent to $P_{\theta_2}(X)$, since $\otherf(X,Z,H)$
no longer contains information about $X$. This illustrates
that when the noise is large, the reconstruction distribution (parametrized
by $\theta_2$) will need to have the expressive power to represent
multiple modes. Otherwise, the reconstruction will tend to capture 
an average output, which would visually look like a fuzzy combination
of actual modes. In the experiments performed here, we have only
considered unimodal reconstruction distributions (with factorized 
outputs), because we expect that even if ${\calP}(X|H)$
is not unimodal, it would be dominated by a single mode when the
noise level is small. However, 
future work should investigate multimodal alternatives.

A related element to keep in mind is that one should pick the
family of conditional distributions $P_{\theta_2}(X|H)$ so that
one can sample from them and one can easily train them when given
$(X,H)$ pairs, e.g., by maximum likelihood.

\subsection{Handling missing inputs or structured output}
\label{sec:missing_inputs}
\svs{2}

In general, a simple way to deal with missing inputs is to clamp the
observed inputs and then run the Markov chain with the constraint
that the observed inputs are fixed and not resampled at each time
step, whereas the unobserved inputs are resampled each time,
\emph{conditioned on the clamped inputs}.

In the context of the GSN described in Section \ref{sec:from-DAE-to-GSN}
using the two distributions
\begin{eqnarray*}
H_{t+1} & \sim & P_{\theta_{1}}(H|H_{t},X_{t})\\
X_{t+1} & \sim & P_{\theta_{2}}(X|H_{t+1})
\end{eqnarray*}
we need to make some adjustments to $P_{\theta_{2}}(X|H_{t+1})$ to
be able to sample $X$ conditioned on some of its components being
clamped. We also focus on the case where there are no connections
between the $H_t \rightarrow H_{t+1}$. That is, we study the
more basic situation where we train an denoising auto-encoder
instead of a GSN that has connections between the hidden units.

Let $\mathcal{S}$ be a set of values that $X$ can take. For example, $\mathcal{S}$
can be a subset of the units of $X$ that are fixed to given values.
We can talk about clamping $X\in \mathcal{S}$, or just ``clamping $\mathcal{S}$''
when the meaning is clear. In order to sample from a distribution
with clamped $\mathcal{S}$, we need to be able to sample from
\begin{eqnarray*}
H_{t+1} & \sim & P_{\theta_{1}}(H|X_{t})\\
X_{t+1} & \sim & P_{\theta_{2}}(X|H_{t+1},X\in S).
\end{eqnarray*}
This notation might be strange at first, but it's as legitimate as
conditioning on $0<X$ when sampling from any general distribution.
It involves only a renormalization of the resulting distribution
$P_{\theta_{2}}(X|H_{t+1},X\in \mathcal{S})$.

In a general scenario with two conditional distributions
$(P_{\theta_{1}},P_{\theta_{2}})$
playing the roles of $f(x|h)$ and $g(h|x)$, i.e. the encoder and
decoder, we can make certain basic assumptions so that the asymptotic
distributions of $(X_{t},H_{t})$ and $(X_{t},H_{t+1})$ both exist.
There is no reason to think that those two distributions are the same,
and it is trivial to construct counter-examples where they differ
greatly.

However, when we train a DAE with infinite capacity,
Proposition \ref{prop:markov-chain-basic-gsn} shows
that the optimal solution leads to those two joints being the same.
That is, the two trained conditional distributions $f(h|x)$ and $g(x|h)$
are \emph{mutually compatible}. They form a single joint
distribution over $(X,H)$. We can sample from it by the usual Gibbs
sampling procedure. Moreover, the marginal distribution over $X$
that we obtain will match that of the training data. This is the motivation
for Proposition \ref{prop:markov-chain-basic-gsn}.

Knowing that Gibbs sampling produces the desired joint distribution
over $(X,H)$, we can now see how it would be possible to sample from
$(X,H)|(X\in \mathcal{S})$ if we are able to sample from $f(h|x)$ and $g(x|h,x\in \mathcal{S})$.
Note that it might be very hard to sample from $g(x|h,x\in \mathcal{S})$, depending
on the particular model used. We are not making any assumption on
the factorization of $g(x|h)$, much like we are not making any assumption
on the particular representation (or implementation) of $g(x|h)$.

In section \ref{sec:a-note-about-consistency-from-Schweitzer}
we address a valid concern about the possibility
that, in a practical setting, we might not train $g(x|h)$ to achieve
an exact match the density of $X|H$. That $g(x|h)$ may be very close
to the optimum, but it might not be able to achieve it due to its
finite capacity or its particular parametrization. What does that
imply about whether the asymptotic distribution of the Markov chain
obtained experimentally compared to the exact joint $(X,H)$ ?

We deal with this issue in the same way as we dealt with it
when it arose in the context of Theorem \ref{thm:noisy-reconstruction}.
The best that we can do
is to refer to Theorem \ref{thm:schweitzer_inequality} and rely
on an argument made in the context of discrete states that would
closely approximate our situation (which is in either discrete or
continuous space).

Our Markov chain is homogeneous because it does not change with time.
It can be made irreducible my imposing very light constraints on
$f(h|x)$ so that $f(h|x)>0$ for all $(x,h)$. This happens automatically
when we take $f(h|x)$ to be additive Gaussian noise (with fixed parameters)
and we train only $g(x|h)$. In that case, the optimum $g(x|h)$ will
assign non-zero probability weight on all the values of $x$.

We cannot guarantee that a non-optimal $g(x|h)$ will not be broken
in some way, but we can often get $g(x|h)$ to be non-zero by selecting
a parametrized model that cannot assign a probability of exactly zero to an $x$.
Finally, to use Theorem \ref{thm:schweitzer_inequality} we need to
have that the constant $\left\Vert Z\right\Vert _{\infty}$ from that
Theorem \ref{thm:schweitzer_inequality} to be non-zero.
This is a bit more complicated to enforce, but it is something
that we will get if the transition matrix stays away from the identity matrix.
That constant is zero when the chain is close to being degenerate.

Theorem \ref{thm:schweitzer_inequality} says that,
with those conditions verified, we have
that an arbitrarily good $g(x|h)$ will lead to an arbitrarily good
approximation of the exact joint $(X,H)$.

Now that we know that this approach is grounded in sound theory, it
is certainly reasonable to try it in experimental settings in which
we are not satisfying all the requirements, and see if the results
are useful or not. We would refer the reader to our experiment
shown in Figure \ref{fig:samples+inpainting_small}
where we clamp certain units and resample the rest.

To further understand the conditions for obtaining the appropriate
conditional distributions on some of the visible inputs when others
are clamped, we consider below sufficient and necessary conditions
for making the stationary distribution of the clamped chain
correspond to the normalized distribution (over the allowed values)
of the unclamped chain.

\begin{proposition}
\label{prop:clamping}

Let $f(h|x)$ and $g(x|h)$ be the encoder and decoder functions
such that they are mutually compatible (i.e. they represent a
single joint distribution for which we can sample using Gibbs sampling).
Let $\pi(X,H)$ denote that joint.

Note that this happens when we minimize
\[
\mathbb{E}_{X} \left[ \log \int g(x|h)f(h|x) dh \right]
\]
or when we minimize the walkback loss (see Proposition \ref{prop:walkback}).

Let $\mathcal{S} \subseteq \mathcal{X}$ be a set of values that $X$ can take
(e.g. some components of $X$ can be assigned certain fixed values),
and such that $\mathbb{P}(X\in\mathcal{S})>0$.
Let $\pi(x | x \in \mathcal{S})$ denote the conditional distribution of $\pi(X,H)$
on which we marginalize over $H$ and condition on $X \in \mathcal{S}$. That is,
\[
\pi(x | x \in \mathcal{S}) \propto \pi(x) \mathbb{I}(x \in \mathcal{S}) = 
\frac{\int_h \pi(x,h) \mathbb{I}(x \in \mathcal{S}) dh}
{\int_x \int_h \pi(x,h) \mathbb{I}(x \in \mathcal{S}) dh dx}.
\]

Let $g(x|h,x \in \mathcal{S})$ denote a restriction of the decoder function
that puts probability weight only on the values of $x\in\mathcal{S}$. That is,
\[
g(x|h,x \in \mathcal{S}) \propto g(x|h) \mathbb{I}(x \in \mathcal{S}).
\]

\textbf{If} we start from some $x_0 \in \mathcal{S}$ and we run a Markov chain
by alternating between $f(h|x)$ and $g(x|h,x \in\mathcal{S})$,
\textbf{then} the asymptotic distribution of that chain with respect to $X$
will be the same as $\pi(x | x \in \mathcal{S})$.

\end{proposition}

\svs{2}
\subsection{General conditions for clamping inputs}
\label{sec:missing_inputs_condition}
\svs{2}

In the previous section we gave a \emph{sufficient} condition for ``clamping $\mathcal{S}$''
to work in the context of a Markov chain based on an encoder
distribution with density $f(h|x)$ and a decoder distribution with density $g(x|h)$.

In this section, we will give a \emph{sufficient} and \emph{necessary} condition
on the sufficient and necessary conditions for handling
missing inputs by clamping observed inputs.
\begin{proposition}
\label{prop:clamp_sufficient}
Assume we have an \textit{ergodic} Markov chain with transition operators
having density $f(h | x)$ and $g(x | h)$.
Its unique stationary distribution is $\pi(x, h)$ over $\cal{X}\times\cal{H}$ which satisfies:
\begin{equation}
\nonumber \int_{\cal{X}\times\cal{H}}\pi(x, h)
				f(h' | x)g(x' | h') dxdh= \pi(x', h').
\end{equation}
Assume that we start from $(X_0, H_0) = (x_0, h_0)$ where $x_0 \in \cal{S}$,
$\cal{S} \subseteq \cal{X}$ ($\cal{S}$ can be considered as a constraint over $X$)
and we sample $(X_{t+1}, H_{t+1})$ by first sampling $H_{t+1}$ with 
encoder $f(H_{t+1}|X_{t})$ and then sampling $X_{t+1}$ with decoder $g(X_{t+1}|H_{t+1}, X_{t+1}\in \cal{S})$,
the new stationary distribution we reach is $\pi_{\cal{S}}(x, h)$.

Then a sufficient condition for
\[
\pi_{\cal{S}}(x) = \pi(x|x \in S)
\]
is for $\pi(x|x \in S)$ to satisfy
\begin{equation}
\label{th3_cond}
\int_{\cal{S}} \pi(x | x \in {\cal{S}})f(h'|x)dx = \pi(h'|x\in\cal{S})
\end{equation}

where $\pi(x | x \in {\cal{S}})$ and $\pi(h'|x\in\cal{S})$ are conditional distributions 
\begin{equation}
\nonumber \pi(x | x \in {\cal{S}}) =\frac{\pi(x)}{\int_{\cal{S}}\pi(x')dx'},\quad
\pi(h' | x \in {\cal{S}}) =\frac{\int_{\cal{S}}\pi(x, h')dx}{\int_{\cal{S}\times{\cal{H}}}\pi(x, h)dxdh}.
\end{equation}

\end{proposition}
\begin{proof}
Based on the assumption that the chain is ergodic, we have that
$\pi_{\cal{S}}(X, H)$ 
is the unique distribution satisfying
\begin{equation}
\label{th3}
\int_{\cal{S}\times\cal{H}}\pi_{\cal{S}}(x, h)f(h'|x) g(x'|h', x'\in{\cal{S}})dxdh = \pi_{\cal{S}}(x', h').
\end{equation}
Now let us check if $\pi(x, h|x \in \cal{S})$ satisfies the equation above.

The Markov chain described in the statement of the Theorem
is defined by looking at the slices $(X_t, H_t)$.
This means that, by construction, the conditional density $\pi(x|h)$
is just given by $g(x|h)$.

This relation still holds even if we put the $\cal{S}$ constraint on $x$ 
\begin{equation}
\nonumber g(x'|h', x'\in{\cal{S}}) = \pi(x'|h', x'\in\cal{S}).
\end{equation}
Now if we substitute $\pi_{\cal{S}}(x, h)$ by $\pi(x, h|x \in \cal{S})$ in Equation \ref{th3},
the left side of Equation \ref{th3} becomes 
\begin{eqnarray}
\nonumber 
&&\int_{\cal{S}\times {\cal{H}}} \pi(x, h|x \in {\cal{S}}) f(h'|x) \pi(x'|h', x'\in{\cal{S}})dxdh\\
\nonumber
&=&\pi(x'|h', x'\in{\cal{S}} )\int_{\cal{S}}(\int_{\cal{H}} \pi(x, h|x \in {\cal{S}})dh) f(h'|x)dx\\
\nonumber
&=&\pi(x'|h', x'\in{\cal{S}} )\int_{\cal{S}}\pi( x|x \in {\cal{S}}) f(h'|x)dx\\
\nonumber
&=&\pi(x'|h', x'\in{\cal{S}} )\pi(h'| x \in {\cal{S}}) \hspace{2em} \textrm{(using Equation \ref{th3_cond})}  \\
\nonumber
&=&\pi(x'|h', x'\in{\cal{S}} )\pi(h'| x' \in {\cal{S}})\\
\nonumber &=&\pi(x',h'| x'\in{\cal{S}} ).
\end{eqnarray}
This shows that $\pi(x, h|x\in{\cal{S}})$ satisfies Equation \ref{th3}.
Due to the ergodicity of the chain, the distribution $\pi_{\cal{S}}(x,h)$
that satisfies Equation \ref{th3} is unique,
so we have $\pi_{\cal{S}}(x,h) =  \pi(x, h|x\in{\cal{S}})$.
By marginalizing over $h$ we get
\begin{equation}
\nonumber \pi_{\cal{S}}(x) =  \pi(x|x\in{\cal{S}}).
\end{equation}
\end{proof}
Proposition \ref{prop:clamp_sufficient} gives a sufficient condition
for dealing missing inputs by clamping observed inputs.
Note that this condition is weaker than the \textit{mutually compatible} condition
discussed in Section \ref{sec:missing_inputs}.
Furthermore, under certain circumstances, this sufficient condition becomes necessary,
and we have the following proposition :
\begin{proposition}
\label{prop:clamp_necessary}
Assume that the Markov chain in Proposition \ref{prop:clamp_sufficient} has finite discrete state space for both $X$ and $H$.
The condition in Equation \ref{th3_cond} in Proposition \ref{prop:clamp_sufficient} becomes a necessary condition when
all discrete conditional distributions $g(x|h, x \in \cal{S})$ are linear independent.
\end{proposition}
\begin{proof}
We follow the same notions in Proposition \ref{prop:clamp_sufficient} and now we have $\pi_{\cal{S}}(x) = \pi(x|x \in S)$.
Because $\pi_{\cal{S}}(x)$ is the marginal of the stationary distribution
reached by alternatively sampling with encoder $f(H|X)$ and decoder $g(X|H, X \in {\cal{S}})$, 
we have that $\pi(x|x \in {\cal{S}})$ satisfies
\begin{equation}
\nonumber \int_{\cal{S}}\pi(x| x\in {\cal{S}})(\int_{\cal{H}}f(h'|x) \pi(x'|h', x'\in{\cal{S}})dh')dx = \pi(x'|x'\in \cal{S})
\end{equation}
which is a direct conclusion from Equation \ref{th3} when considering the fact that
$\pi_{\cal{S}}(x) = \pi(x|x \in S)$ and $g(x'|h', x'\in{\cal{S}}) = \pi(x'|h', x'\in{\cal{S}})$.
If we re-arrange the integral of equation above, we get:
\begin{equation}
\label{th3_c1_2}
\int_{\cal{H}} \pi(x'|h', x'\in{\cal{S}})(\int_{\cal{S}}\pi(x| x\in {\cal{S}})f(h'|x)dx)dh' = \pi(x'|x'\in \cal{S}).
\end{equation}
Note that $\int_{\cal{S}}\pi(x|x\in{\cal{S}})f(h'|x)dx$ is the same as the left side
of Equation \ref{th3_cond} in Proposition \ref{prop:clamp_sufficient} and it can be seen as some function  $F(h')$ satisfying
$\int_{\cal{H}}F(h')dh' = 1$. Because we have considered a GSN over a finite discrete state space
${\cal{X}} = \{x_1, \cdots, x_N\}$ and ${\cal{H}} = \{h_1, \cdots, h_M\}$,
the integral in Equation \ref{th3_c1_2} becomes the linear matrix equation
\begin{equation}
\nonumber \mathbf{G}\cdot\mathbf{F}=\mathbf{P}_x,
\end{equation}
where $\mathbf{G}(i,j) = g(x'_i|h'_j, x'\in{\cal{S}}) = \pi(x'_i|h'_j, x'\in{\cal{S}})$, $\mathbf{F}(i) = F(h'_i)$
and $\mathbf{P}_x(i) =\pi(x_i'|x'\in{\cal{S}}) $.
In other word, $\mathbf{F}$ is a solution of the linear matrix equation
\begin{equation}
\nonumber \mathbf{G}\cdot\mathbf{Z}=\mathbf{P}_x.
\end{equation}
From the definition of $\mathbf{G}$ and $\mathbf{P}_x$, it is obvious that 
$\mathbf{P}_h$ is also a solution of this linear matrix equation, if $\mathbf{P}_h(i) = \pi(h_i'|x'\in{\cal{S}}) $.
Because all discrete conditional distributions $g(x|h, x\in \cal{S})$ are linear independent,
which means that all the column vectors of $\mathbf{G}$ are linear independent, then this
linear matrix equation has no more than one solution. Since $\mathbf{P}_h$ is the solution,
we have $\mathbf{F} = \mathbf{P}_h$, equivalently in integral form
\begin{equation}
\nonumber F(h') = \int_{\cal{S}} \pi(x | x \in {\cal{S}})f(h'|x)dx = \pi(h'|x\in\cal{S})
\end{equation}
which is the condition Equation \ref{th3_cond} in Proposition \ref{prop:clamp_sufficient}.

\end{proof}
Proposition \ref{prop:clamp_necessary} says that at least in discrete finite state space,
if the $g(x|h, x\in \cal{S})$ satisfies some reasonable condition like linear independence, then
along with Proposition \ref{prop:clamp_sufficient}, the condition in Equation \ref{th3_cond} is 
the necessary and sufficient condition for handling missing inputs by clamping the observed
part for at least one subset $\cal{S}$. If we want this result to hold
for any subset $\cal{S}$, we have the following proposition:
\begin{proposition}
\label{prop:clamp_any}
If the condition in Equation \ref{th3_cond} in Proposition \ref{prop:clamp_sufficient} holds for any subset of $\cal{S}$ that $\cal{S}\subseteq \cal{X}$,
then we have 
\begin{equation}
\nonumber f(h'|x) = \pi(h'|x)
\end{equation}
In other words, $f(h|x)$ and $g(x|h)$ are two conditional distributions 
marginalized from a single joint distribution $\pi(x, h)$.
\end{proposition}

\begin{proof}
Because $\cal{S}$ can be any subset of $\cal{X}$, of course that $\cal{S}$
can be a set which only has one element $x_0$, i.e., ${\cal{S}} = \{x_0\}$.
Now the condition in Equation \ref{th3_cond} in Proposition \ref{prop:clamp_sufficient} becomes 
\begin{equation}
\nonumber 1 \cdot f(h'|x=x_0) = \pi(h'|x=x_0).
\end{equation}
Because $x_0$ can be an arbitrary element in $\cal{X}$, we have
\begin{equation}
\nonumber f(h'|x) = \pi(h'|x), \quad {\rm or} \quad f(h|x) = \pi(h|x). 
\end{equation}
Since from Proposition \ref{prop:clamp_sufficient} we already know
that $g(x|h)$ is $\pi(x|h)$, we have that
$f(h|x)$ and $g(x|h)$ are \textit{mutually compatible}, that is, 
they are two conditional distributions 
obtained by normalization from a single joint distribution $\pi(x, h)$.
\end{proof}
According to Proposition \ref{prop:clamp_any},
if condition in Equation \ref{th3_cond} holds for any subset $\cal{S}$,
then $f(h|x)$ and $g(x|h)$ must be \textit{mutually compatible} to the single joint
distribution $\pi(x,h)$. 

\subsection{Dependency Networks as GSNs}
\label{sec:dependency-nets}
\svs{2}
Dependency networks~ \mbox{\citep{HeckermanD2000}} are models in which one
estimates conditionals $P_i(x_i | x_{-i})$, where $x_{-i}$ denotes $x \setminus x_i$,
i.e., the set of variables other than the $i$-th one, $x_i$. Note that
each $P_i$ may be parametrized separately, thus not guaranteeing that
there exists a joint of which they are the conditionals. Instead of the
ordered pseudo-Gibbs sampler defined in~ \mbox{\citet{HeckermanD2000}}, which
resamples each variable $x_i$ in the order $x_1, x_2, \ldots$, we can view
dependency networks in the GSN framework by defining a proper Markov chain
in which at each step one randomly chooses which variable to resample. The
corruption process therefore just consists of $H=f(X,Z)=X_{-s}$ 
where $X_{-s}$ is the complement of $X_{s}$, with $s$
a randomly chosen subset of elements of $X$ (possibly constrained to be
of size 1).  Furthermore, we parametrize the reconstruction
distribution as $P_{\theta_2}(X=x|H) = \delta_{x_{-s}=X_{-s}}P_{\theta_2,s}(X_s=x_s | x_{-s})$ 
where the estimated conditionals
$P_{\theta_2,s}(X_s=x_s | x_{-s})$ are not constrained to be consistent
conditionals of some joint distribution over all of $X$.

\begin{proposition}
If the above GSN Markov chain has a stationary distribution, then
the dependency network defines a joint distribution (which is that
stationary distribution), which does not have to be known in closed
form. Furthermore, if the conditionals $P(X_s | X_{-s})$ are consistent estimators
of the ground truth conditionals, then that stationary distribution 
is a consistent estimator of the ground truth joint distribution.
\vs{1}
\end{proposition}
The proposition can be proven by immediate application of Proposition~\ref{prop:markov-chain-basic-gsn}
with the above particular GSN model definitions.

This joint stationary
distribution can exist even if the conditionals are not consistent.
To show that, assume that some choice of (possibly inconsistent)
conditionals gives rise to a stationary distribution $\pi$.
Now let us consider the set of all conditionals (not necessarily
consistent) that could have given rise to that $\pi$.
Clearly, the conditionals derived from $\pi$ by Bayes rule are part of that
set, but there are infinitely many others (a simple counting
argument shows that the fixed point equation of $\pi$ introduces
fewer constraints than the number of degrees of freedom that
define the conditionals). To better understand why the ordered
pseudo-Gibbs chain does not benefit from the same properties, we can consider an extended case
by adding an extra component of the state $X$, being
the index of the next variable to
resample. In that case, the Markov chain associated with the ordered pseudo-Gibbs procedure
would be periodic, thus violating the ergodicity assumption
of the theorem. However, by introducing randomness in the choice
of which variable(s) to resample next, we obtain aperiodicity
and ergodicity, yielding as stationary distribution
a mixture over all possible resampling orders. These results also
show in a novel way (see e.g. ~\citet{Hyvarinen-2006-small} for earlier results)
that training by pseudolikelihood or generalized pseudolikelihood
provides a consistent estimator of the associated joint, so long as the GSN
Markov chain defined above is ergodic. This result can
be applied to show that the multi-prediction deep Boltzmann machine (MP-DBM)
training procedure introduced by~\citet{Goodfellow-et-al-NIPS2013} also
corresponds to a GSN. This has been exploited in order to obtain
much better samples using the associated GSN Markov chain than
by sampling from the corresponding DBM~\citep{Goodfellow-et-al-NIPS2013}.
Another interesting conclusion that one can draw from that paper
and its GSN interpretation is that state-of-the-art classification
error can thereby be obtained: 0.91\% on MNIST without fine-tuning (best 
comparable previous DBM results was well above 1\%) and 10.6\% on permutation-invariant
NORB (best previous DBM results was 10.8\%).

\svs{2}
\section{Experimental results}
\label{sec:experiment}
\svs{2}

The theoretical results on Generative Stochastic Networks (GSNs)
open for exploration a large class of possible
parametrizations and training procedures which share the property that they can capture the
underlying data distribution through the GSN Markov chain. What
parametrizations will work well?  Where and how should one inject noise to best balance fast mixing with making the implied conditional easy to model? We present
results of preliminary experiments with specific selections for each of these choices, but
the reader should keep in mind that the space of possibilities is vast.

We start in Section \ref{sec:walkback_experiment} with results involving
GSNs without latent variables (denoising auto-encoders in Section \ref{sec:DAE-model-prob-density} and
the walkback algorithm presented in Section \ref{sec:walkback-description}).
Then in Section \ref{sec:gsn_experiment} we proceed with experiments
related to GSNs with latent variables (model described in Section \ref{sec:from-DAE-to-GSN}). 
Section \ref{sec:swb_experiment} extends experiments of the walkback algorithm 
with the scaling factors discussed in 
Section \ref{sec:swb}. 
A Theano\footnote{http://deeplearning.net/software/theano/}
 \citep{bergstra+al:2010-scipy-short} 
implementation
is available\footnote{https://github.com/yaoli/GSN}, including the links of datasets.
\svs{2}
\subsection{Experimental results regarding walkback in DAEs}
\label{sec:walkback_experiment}
\svs{2}

We present here an experiment performed with a non-parametric estimator on two types of data
and an experiment done with a parametric neural network on the MNIST dataset.

{\bf Non-parametric case.}
The mathematical results presented here
apply to any denoising training criterion where the reconstruction
loss can be interpreted as a negative log-likelihood. This
remains true whether or not the denoising machine $P(X|\tilde{X})$
is parametrized as the composition of an encoder and decoder.
This is also true of the asymptotic estimation results
in~\citet{Alain+Bengio-ICLR2013}.
We experimentally validate the above theorems in a case where the asymptotic limit
(of enough data and enough capacity) can be reached, i.e., in a
low-dimensional non-parametric setting.
Fig.~\ref{fig:histogram} shows the distribution
recovered by the Markov chain for {\bf discrete data}
with only 10 different values.  The conditional
$P(X|\tilde{X})$ was estimated by multinomial models and maximum likelihood (counting) from 5000
training examples. 5000 samples were generated from the chain to
estimate the asymptotic distribution $\pi_n(X)$.  For 
{\bf continuous data}, Figure~\ref{fig:histogram} also
shows the result of 5000 generated samples and 500 original training examples
with $X \in \mathbb{R}^{10}$, with scatter plots of pairs of
dimensions. The estimator is also non-parametric (Parzen density estimator
of $P(X|\tilde{X})$). 
\begin{figure}[htb]
\centering
\vs{8}
\includegraphics[width=0.40\textwidth]{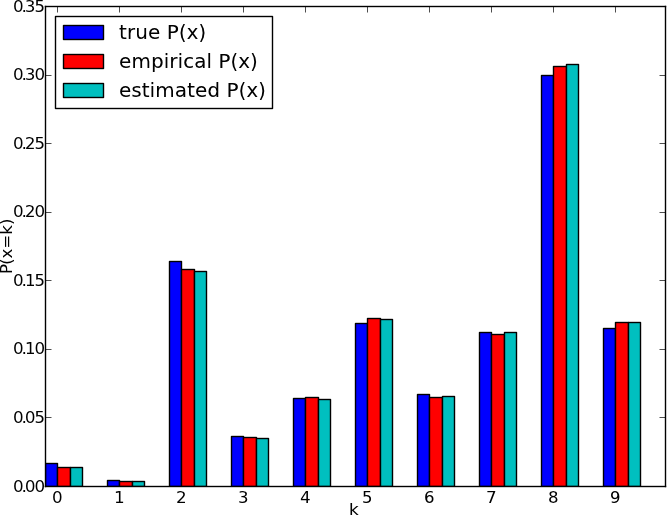}
\includegraphics[width=0.40\textwidth]{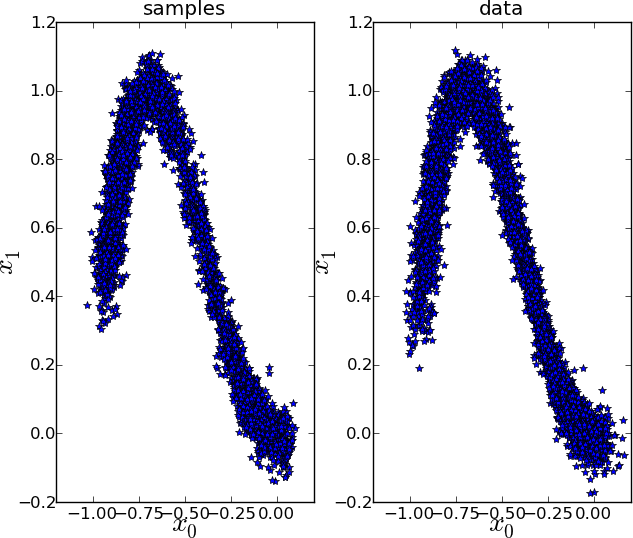}
\vs{2}
\includegraphics[width=0.40\textwidth]{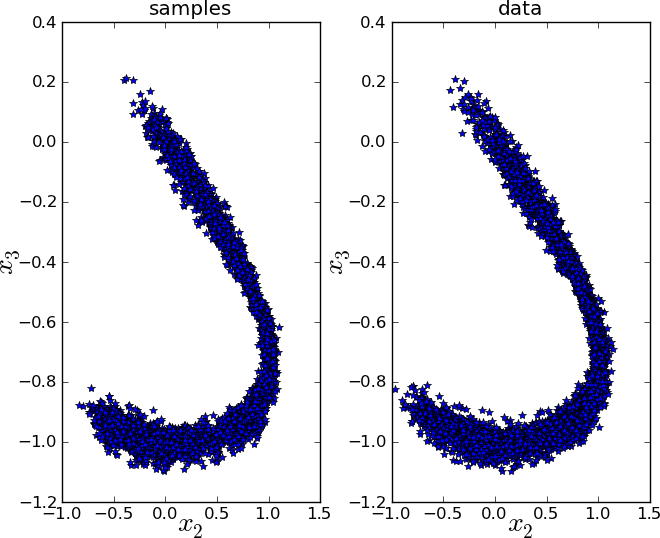}
\includegraphics[width=0.40\textwidth]{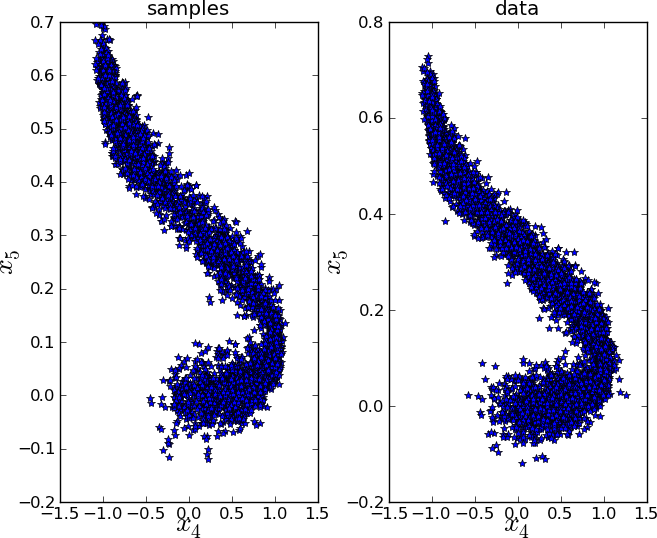}
\vs{1}
\caption{\sl Top left: histogram of a data-generating distribution (true, blue),
the empirical distribution (red), and the estimated distribution using
a denoising maximum likelihood estimator. Other figures: pairs of variables
(out of 10) showing the training samples and the model-generated samples.}
\label{fig:histogram}
\vs{4}
\end{figure}

{\bf MNIST digits.}  We trained a DAE on the
binarized MNIST data (thresholding at 0.5). 
The 784-2000-784 auto-encoder is trained for 200 epochs with the 50000 training examples and salt-and-pepper noise
(probability 0.5 of corrupting each bit, setting it to 1 or 0 with
probability 0.5). It has 2000 tanh hidden units and is trained by minimizing cross-entropy loss,
i.e., maximum likelihood on a factorized Bernoulli reconstruction distribution.
With walkback training, a chain of 5
steps was used to generate 5 corrupted examples for each training
example. Figure~\ref{fig:MNIST} shows samples generated with and without walkback.
The quality of the samples was also estimated quantitatively by measuring
the log-likelihood of the test set under a non-parametric density
estimator $\hat{P}(x)={\rm mean}_{\tilde{X}} P(x|\tilde{X})$
constructed from 10,000 consecutively generated samples
($\tilde{X}$ from the Markov chain). The expected value of $\mathbb{E}[\hat{P}(x)]$
over the samples can be shown~\citep{Bengio+Yao-arxiv-2013} to be
a lower bound (i.e. conservative estimate) of the true (implicit) model density $P(x)$.
The test set log-likelihood bound
was not used to select among model architectures, but visual inspection of
samples generated did guide the preliminary search reported here.
Optimization hyper-parameters (learning rate, momentum, and
learning rate reduction schedule) were selected based on the 
training objective. We compare against a state-of-the-art RBM~\citep{NECO_cho_2013_enhanced}
with an AIS log-likelihood estimate of -64.1 (AIS estimates tend to be optimistic).
We also drew samples from the RBM and applied the same estimator (using the mean of the RBM's $P(x|h)$ with $h$
sampled from the Gibbs chain), and obtained a log-likelihood non-parametric bound of -233,
skipping 100 MCMC steps between samples (otherwise
numbers are very poor for the RBM, which mixes poorly).
The DAE log-likelihood bound
with and without walkback is respectively -116 and -142,
confirming visual inspection suggesting that
the walkback algorithm produces less spurious samples. However, the
RBM samples can be improved by a spatial blur. By tuning the amount of
blur (the spread of the Gaussian convolution), we obtained a bound of -112
for the RBM. Blurring did not help the auto-encoder. 

\begin{figure}[htb]
\centering
\vs{2}
\hspace*{-1mm}\includegraphics[width=0.5\textwidth]{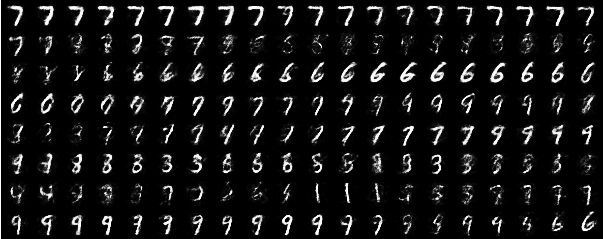} \hspace*{-1mm} \includegraphics[width=0.5\textwidth]{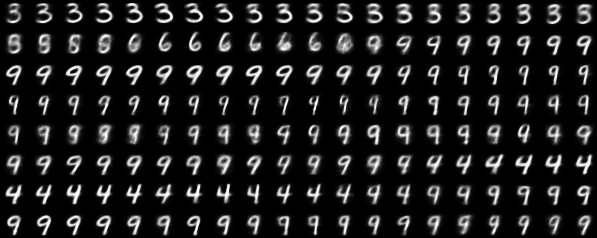}
\vs{5}
\caption{\sl Successive samples generated by Markov chain associated with
the trained DAEs
according to the plain sampling scheme (left) and walkback sampling scheme (right).
There are less ``spurious'' samples with the walkback algorithm.}
\label{fig:MNIST}
\vs{2}
\end{figure}

\svs{2}
\subsection{Experimental results for GSNs with latent variables}
\label{sec:gsn_experiment}
\svs{2}

We propose here to explore families of
parametrizations which are similar to existing deep stochastic architectures
such as the Deep Boltzmann Machine (DBM)~\citep{Salakhutdinov+Hinton-2009-small}.
Basically, the
idea is to construct a computational graph that is similar to the
computational graph for Gibbs sampling or variational inference
in Deep Boltzmann Machines. However, we have
to diverge a bit from these architectures in order to accommodate the
desirable property that it will be possible to back-propagate
the gradient of reconstruction log-likelihood with respect to the
parameters $\theta_1$ and $\theta_2$. Since the gradient of a binary
stochastic unit is 0 almost everywhere, we have to consider related
alternatives. An interesting source of inspiration regarding this
question is a recent paper on estimating or propagating gradients
through stochastic neurons~\citep{Bengio-arxiv2013}.
Here we consider the following stochastic non-linearities:
$h_i=\eta_{\rm out}+\tanh(\eta_{\rm in} + a_i)$
where $a_i$ is the linear activation for unit $i$ (an affine transformation
applied to the input of the unit, coming from the layer below, the layer
above, or both) and $\eta_{\rm in}$ and $\eta_{\rm out}$ are 
zero-mean Gaussian noises.

To emulate a sampling procedure similar to Boltzmann machines 
in which the filled-in missing values can depend on the
representations at the top level, the computational graph allows
information to propagate both upwards (from input to higher levels)
and downwards, giving rise to the
computational graph structure illustrated in Figure~\ref{fig:comp-graph},
which is similar to that explored for {\em deterministic} recurrent
auto-encoders~\citep{SeungS1998,Behnke-2001,Savard-master-small}. Downward
weight matrices have been fixed to the transpose of corresponding
upward weight matrices.

\begin{figure*}[ht]
\vs{3}
\centering
\begin{minipage}{0.35\textwidth}
\vspace*{-3mm}\hspace*{5mm}
\includegraphics[width=0.8\textwidth]{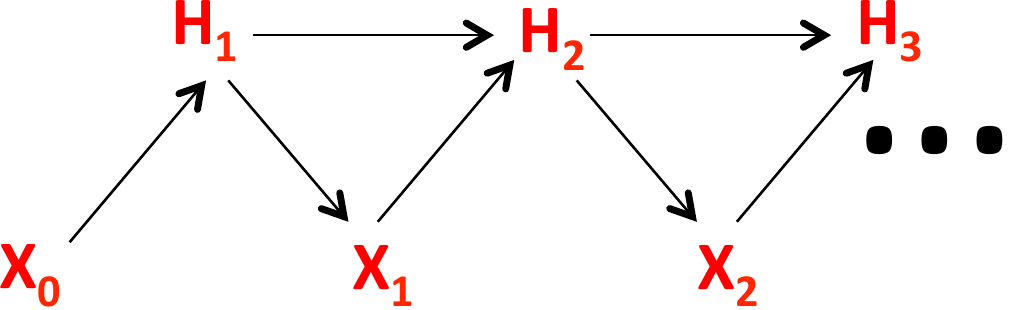} 
\end{minipage} 
\begin{minipage}{0.64\textwidth}
\includegraphics[width=0.99\textwidth]{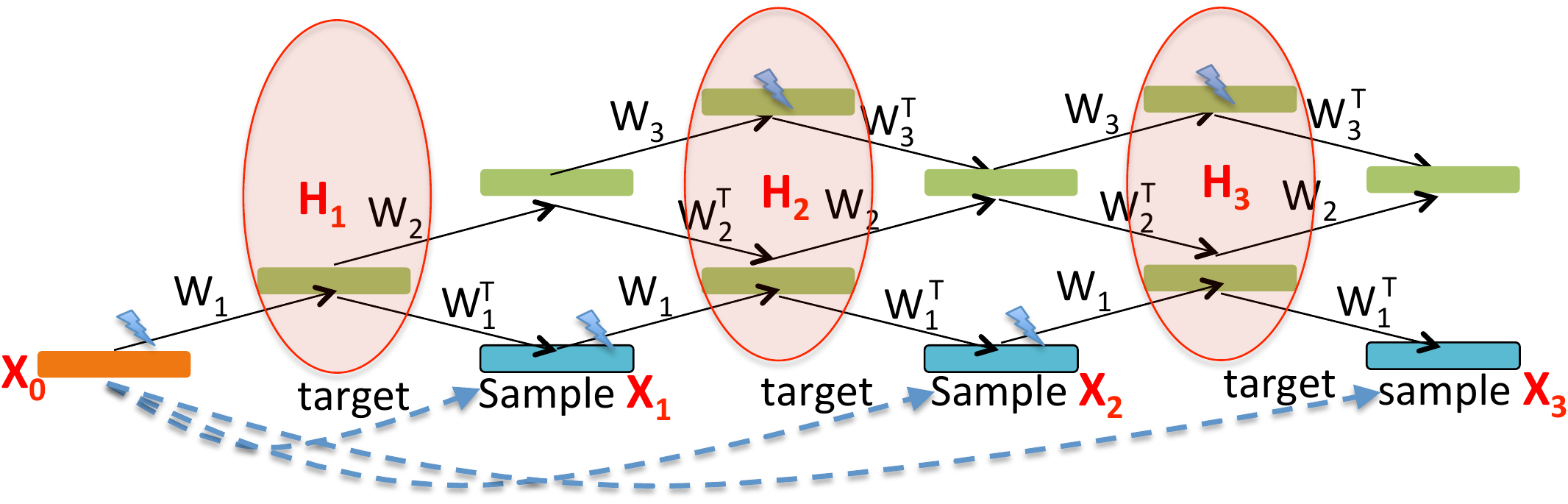}
\end{minipage}
\vs{3}
\caption{{\em Left:} Generic GSN Markov chain with state variables
$X_t$ and $H_t$. {\em Right:} GSN Markov chain inspired by the
unfolded computational graph of the Deep Boltzmann
Machine Gibbs sampling process, but with backprop-able stochastic units
at each layer. The training example $X=x_0$ starts the chain. Either
odd or even layers are stochastically updated at each step. All $x_t$'s
are corrupted by salt-and-pepper noise before entering the graph
(lightning symbol). Each $x_t$ for $t>0$ is obtained by sampling
from the reconstruction distribution for that step, $P_{\theta_2}(X_t|H_t)$. The walkback
training objective is the sum over all steps of log-likelihoods of target
$X=x_0$ under the reconstruction distribution.
In the special case of a unimodal Gaussian reconstruction
distribution, maximizing the likelihood is equivalent to minimizing
reconstruction error; in general one trains to maximum
likelihood, not simply minimum reconstruction error.}
\label{fig:comp-graph}
\vs{3}
\end{figure*}


With the {\em walkback} algorithm, a different reconstruction distribution
is obtained after each step of the short chain started at the training example $X$.
It means that the computational graph from $X$ to a reconstruction
probability at step $k$ actually involves generating intermediate samples as if we were
running the Markov chain starting at $X$. In the experiments, the
graph was unfolded so that $2 D$ sampled reconstructions would be produced,
where $D$ is the depth (number of hidden layers). The training loss is
the sum of the reconstruction negative log-likelihoods (of target $X$) 
over all $2 D$ reconstructions.


Experiments evaluating the ability of the GSN models to generate good samples
were performed on the MNIST dataset and the Toronto Face Database (TFD), following the setup 
in~\citet{Bengio-et-al-ICML2013}.

Theorem \ref{thm:noisy-reconstruction} requires $H_0$ to have the same distribution as $H_1$ (given $X_0$) during training,
and this may be achieved by initializing each training chain with
$H_0$ set to the previous value of $H_1$ when the same example $X_0$ was shown. However, 
it turned out that even with a dumb initialization of $H_0$, good results were obtained
in the experiments below. 

Networks with 2 and 3 hidden layers
were evaluated and compared to regular denoising auto-encoders. The latter has just 1
hidden layer and no state to state transition, i.e., the computational graph can be 
split into separate graphs
for each reconstruction step in the walkback algorithm. They all have $\tanh$
hidden units and pre- and post-activation Gaussian noise of standard
deviation 2, applied to all hidden layers except the first.
In addition, at each step in the
chain, the input (or the resampled $X_t$) is corrupted with salt-and-pepper
noise of 40\% (i.e., 40\% of the pixels are corrupted, and replaced with a
0 or a 1 with probability 0.5). Training is over 100 to 600 epochs at most, with
good results obtained after around 100 epochs, using stochastic gradient descent
(minibatch size of one example). Hidden layer sizes vary
between 1000 and 1500 depending on the experiments, and a learning rate of
0.25 and momentum of 0.5 were selected to approximately minimize the
reconstruction negative log-likelihood. The learning rate is reduced
multiplicatively by $0.99$ after each epoch.  Following~\citet{Breuleux+Bengio-2011},
the quality of the samples
was also estimated quantitatively by measuring the log-likelihood
of the test set under a Parzen density estimator constructed from
10,000 consecutively generated samples (using the real-valued mean-field reconstructions
as the training data for the Parzen density estimator). This can be
seen as a {\em lower bound on the true log-likelihood}, with the bound
converging to the true likelihood as we consider more samples and
appropriately set the smoothing parameter of the Parzen 
estimator.\footnote{However, in this paper, to be consistent with
the numbers given in \citet{Bengio-et-al-ICML2013} we used a Gaussian
Parzen density, which makes the numbers not comparable with the
AIS log-likelihood upper bounds for binarized images reported in other papers
for the same data.}


Results are summarized in Table~\ref{tab:LL}. As in Section~\ref{sec:walkback_experiment}, 
the test set Parzen log-likelihood bound
was not used to select among model architectures, but visual inspection of
generated samples guided this preliminary search.
Optimization hyper-parameters (learning rate, momentum, and
learning rate reduction schedule) were selected based on the 
reconstruction log-likelihood training objective. The Parzen log-likelihood bound obtained
with a two-layer model on MNIST is 214 ($\pm$ standard error of 1.1), while the log-likelihood
bound obtained by a single-layer model (regular denoising auto-encoder, DAE in
the table) is
substantially worse, at -152$\pm$2.2.

In comparison,~\citet{Bengio-et-al-ICML2013} report a log-likelihood bound of -244$\pm$54 for
RBMs and 138$\pm$2 for a 2-hidden layer DBN, using the same setup. We have
also evaluated a 3-hidden layer DBM~\citep{Salakhutdinov+Hinton-2009-small}, using
the weights provided by the author, and obtained a Parzen log-likelihood bound of 32$\pm$2.
See \url{http://www.utstat.toronto.edu/~rsalakhu/DBM.html} for details.

Interestingly, the GSN and the DBN-2 actually perform slightly better than
when using samples directly coming from the MNIST training set, perhaps because
the mean-field outputs we use are more ``prototypical'' samples.

Figure~\ref{fig:samples+inpainting} shows two runs of consecutive samples
from this trained model, illustrating that it mixes quite well (faster
than RBMs) and produces rather sharp digit images. The figure shows
that it can also stochastically complete missing values: the left half
of the image was initialized to random pixels and the right side was clamped
to an MNIST image. The Markov chain explores plausible variations of 
the completion according to the trained conditional distribution.


\begin{figure}[htpb]
\vs{3}
\centering
\includegraphics[width=0.5\linewidth]{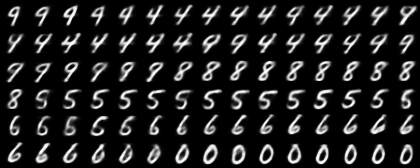} 

\vspace*{1mm}
\includegraphics[width=0.5\linewidth]{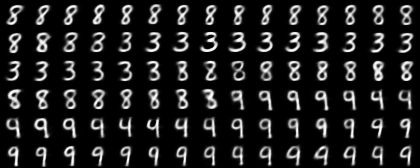}

\vspace*{1.5mm}
\includegraphics[width=0.5\linewidth]{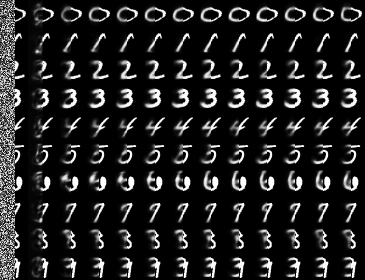}
\vs{5}
\caption{Top: two runs of consecutive samples (one row after the other) generated
from 2-layer GSN model,
showing fast mixing between classes and nice sharp images. Note: only every fourth sample is shown.
Bottom: conditional Markov chain, with the right half of the image clamped to
one of the MNIST digit images and the left half successively resampled, illustrating
the power of the generative model to stochastically fill-in missing inputs.}
\label{fig:samples+inpainting_small}
\vs{5}
\end{figure}

\begin{figure*}[ht]

\centering
\includegraphics[width=0.75\textwidth]{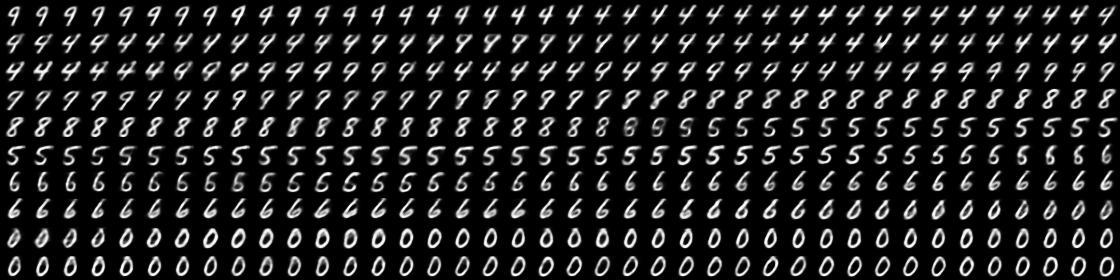} 

\includegraphics[width=0.75\textwidth]{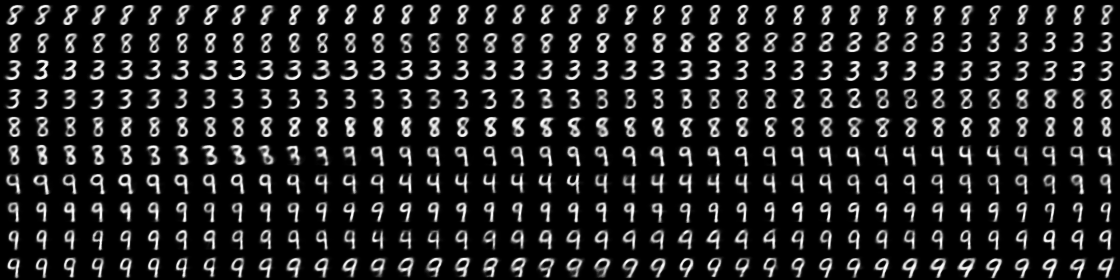}

\vspace*{1.5mm}
\includegraphics[width=.75\textwidth]{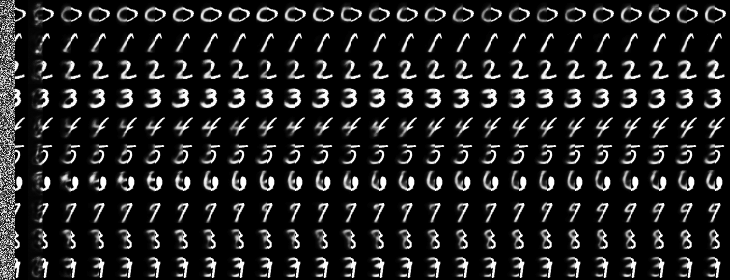}
\caption{These are expanded plots of those in Figure~\ref{fig:samples+inpainting_small}.
{\em Top:} two runs of consecutive samples (one row after the other) generated
from a 2-layer GSN model,
showing that it mixes well between classes and produces nice sharp images. Figure~\ref{fig:samples+inpainting_small} contained only one in every four samples, whereas here every sample is shown.
{\em Bottom:} conditional Markov chain, with the right half of the image clamped to
one of the MNIST digit images and the left half successively resampled, illustrating
the power of the trained generative model to stochastically fill-in missing inputs.
Figure~\ref{fig:samples+inpainting_small} showed only 13 samples in each chain; here we show 26.}
\label{fig:samples+inpainting}
\end{figure*}

\begin{figure*}[ht]
\centering
\includegraphics[width=0.39\textwidth]{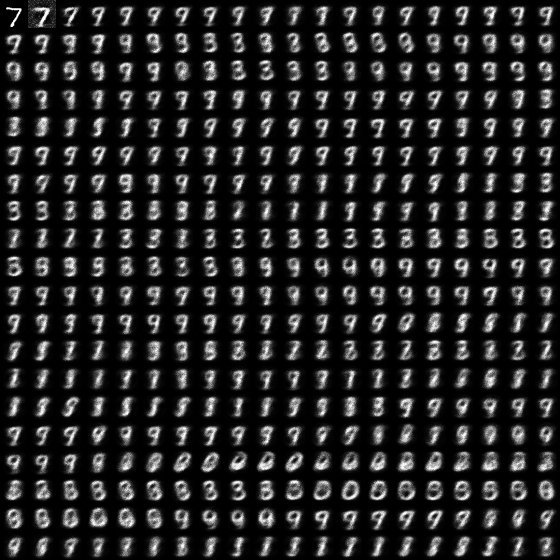}  \vspace*{2mm} %
\includegraphics[width=0.39\textwidth]{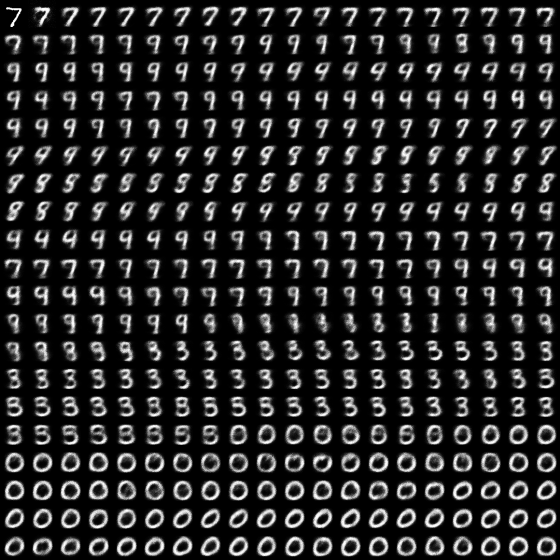}
\caption{Left: consecutive GSN samples obtained after 10 training epochs. Right: GSN
  samples obtained after 25 training epochs. This shows quick convergence to a model that samples well. The samples in
  Figure~\ref{fig:samples+inpainting} are obtained after 600 training
  epochs.}
\label{fig:early-samples}
\end{figure*}

\begin{figure*}[ht]
\centering
\includegraphics[width=1\textwidth]{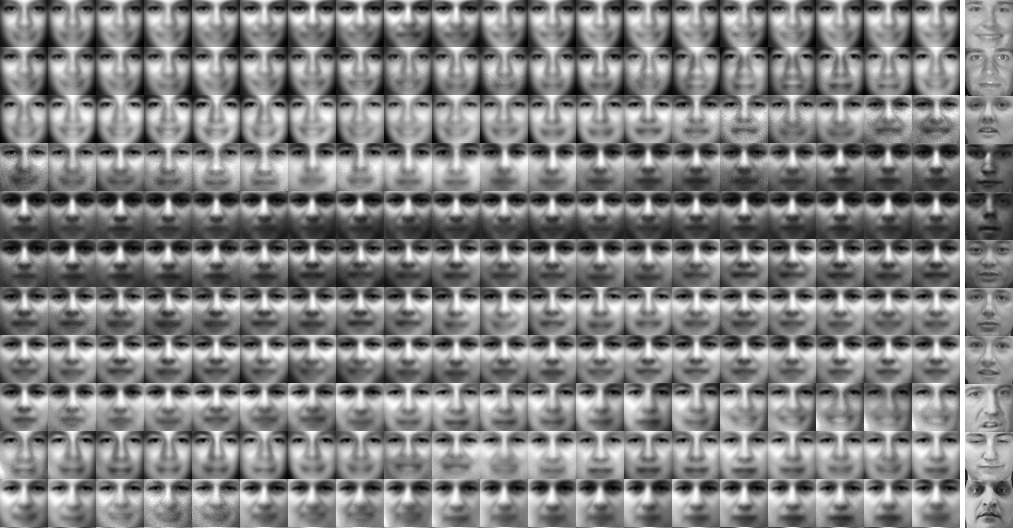}  \vspace*{2mm} %
\caption{Consecutive GSN samples from a model trained on the TFD dataset.  At the end of each row, we show the nearest example from the training set to the last sample on that row to illustrate that the distribution is not merely copying the training set.}
\label{fig:tfd-samples-no-sw}
\end{figure*}




\begin{table}[htpb]
\vs{3}
\caption{Test set log-likelihood lower bound (LL) obtained by a Parzen density estimator
constructed using 10,000 generated samples, for different generative models trained
on MNIST.
The LL is not directly comparable to AIS likelihood estimates
because we use a Gaussian mixture rather
than a Bernoulli mixture to compute the likelihood, but we can compare with
\citet{Rifai-icml2012,Bengio-et-al-ICML2013,Bengio-et-al-NIPS2013} (from which we took the last three columns).
A DBN-2 has 2 hidden layers, a CAE-1 has 1 hidden layer, and a CAE-2 has 2.
The DAE is basically a GSN-1, with no injection of noise inside the network.
\ifLongversion
The last column uses 10,000 MNIST training examples to train the Parzen density estimator.
\fi
}
\label{tab:LL}
\begin{center}
\begin{small}
\begin{sc}
\ifLongversion
\begin{tabular}{lrrrrrr}
\else
\begin{tabular}{lrrrrr} 
\fi
\toprule
\ifLongversion
& GSN-2 & DAE & RBM & DBM-3 & DBN-2 & MNIST\\
Log-likelihood lower bound &  214 & -152 & -244 &  32 & 138 & 24\\
Standard error &  1.1 &  2.2 &   54 &  1.9&  2.0 & 1.6\\
\else
& GSN-2 & DAE & DBN-2 & CAE-1 & CAE-2 \\ 
\midrule
LL  &  214 & 144 & 138 & 68  & 121 \\ 
std.err. &  1.1 & 1.6 &  2.0      & 2.9 & 1.6 \\ 
\fi
\bottomrule
\end{tabular}
\end{sc}
\end{small}
\end{center}
\svs{3}
\end{table}

\subsection{Experimental results for GSNs with the scaling factors for walkbacks}\label{sec:swb_experiment}
We present the experimental results regarding the discussion in Section \ref{sec:swb}. 
Experiments are done on both MNIST and TFD. For TFD, only the unsupervised part of the 
dataset is used, resulting 69,000 samples for train, 15,000 for validation, and 15,000 
for test. The training examples are normalized to have a mean 0 and a standard deviation 1. 

For MNIST the GSNs we used have 2 hidden layers with 1000 tanh units each. 
Salt-and-pepper noise is used to corrupt inputs. We have performed extensive hyperparameter 
search on both the input noise level between 0.3 and 0.7, and the hidden noise level 
between 0.5 and 2.0. The number of walkback steps is also randomly sampled between 2 and 6. 
All the experiments are done with learning the scaling factors, following the 
parameterization in Section \ref{sec:swb_binary}. Following previous experiments, the 
log-probability of the test set is estimated by the same Parzen density estimator on 
consecutive 10,000 samples generated from the trained model. The $\sigma$ parameter in the Parzen 
estimator is cross-validated on the validation set. The sampling is performed 
with $\alpha_1$, the learned scaling factor for the first walkback step. 
The best model achieves a log-likelihood LL=237.44 on 
MNIST test set, which can be compared with the best reported result LL=225 from 
\citet{goodfellow2014generative}.  

On TFD, we follow a similar procedure as in MNIST, but with larger model capacity 
(GSNs with 2000-2000 tanh units) and a 
wider hyperparameter range on the input noise level (between 0.1 and 0.7), 
the hidden noise level (between 0.5 and 5.0), and the number 
of walkback steps (between 2 and 6). For comparison, two types of models are trained, 
one with the scaling factor and one without. The evaluation metric is the same as the one 
used in MNIST experiments. We compute the Parzen density estimation on the first 10,000 
test set examples. The best model without learning the scaling factor results in $LL=1044$, and 
the best model with learning the scaling factor results in 1215 when the scaling factor 
from the first walkback step is used and 1189 when all the scaling factors are used together 
with their corresponding walkback steps. As two further comparisons, using 
the mean over training examples to train the Parzen density estimator results in
$LL=632$, and using the validation set examples to train the Parzen estimator obtains $LL=2029$ (this can be considered as 
an upper bound when the generated samples are almost perfect). 
Figure~\ref{fig:tfd-samples-yes-sw} shows the consecutive 
samples generated with the best model, compared with Figure~\ref{fig:tfd-samples-no-sw}
that is trained without the scaling factor. In addition, Figure~\ref{fig:swb_alpha} shows 
the learned scaling factor for both datasets that confirms the hypothesis on the 
effect of the scaling factors made in
Section~\ref{sec:swb}.
     
\begin{figure*}[ht]
\centering
\includegraphics[width=0.6\textwidth]{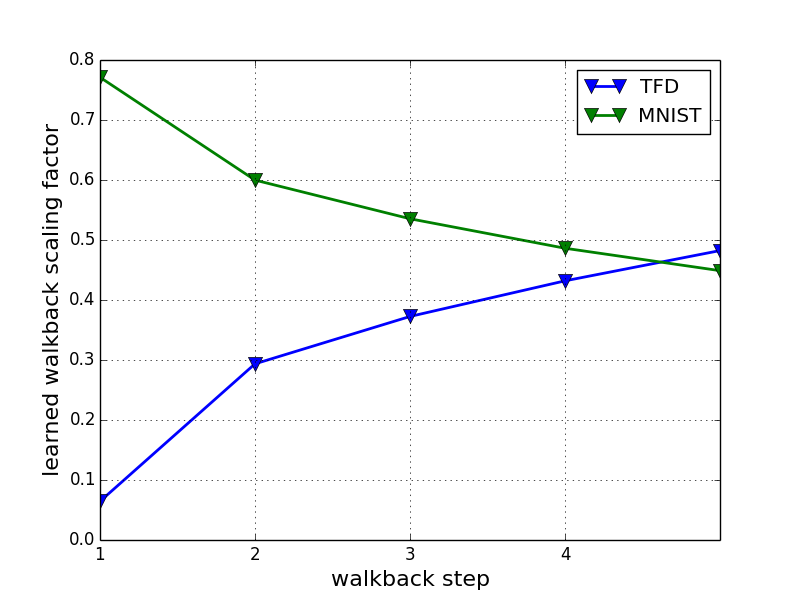}  \vspace*{2mm} %
\caption{Learned $\alpha_k$ values for each walkback step $k$.
Larger values of $\alpha_k$ correspond to \emph{greater} uncertainty for TFD (real-valued) and \emph{less} uncertainty for MNIST (binary), due to the differing methods of parameterization given in Section \ref{sec:swb_binary} and \ref{sec:swb_real}.
Thus, both learned factors reflect the fact that there is greater uncertainty after each consecutive walkback step.
}
\label{fig:swb_alpha}
\end{figure*}

\begin{figure*}[ht]
\centering
\includegraphics[width=0.6\textwidth]{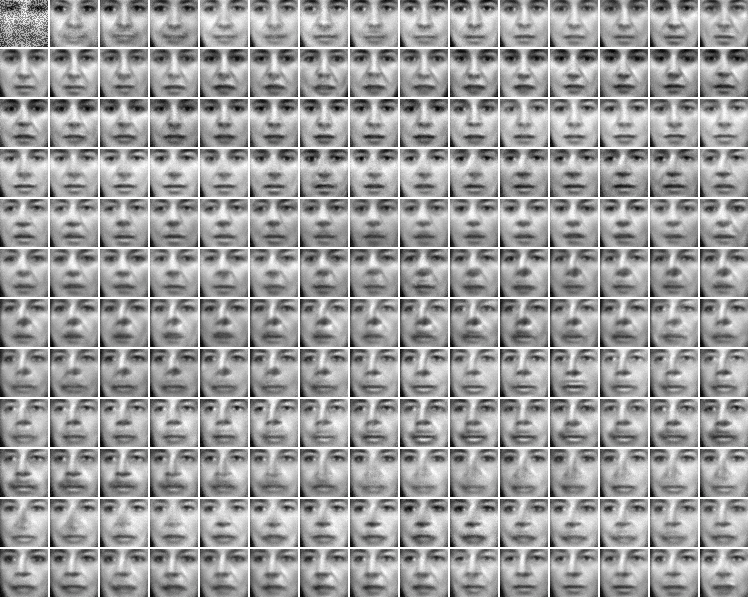}  \vspace*{2mm} %
\caption{Consecutive GSN samples from a model trained on the TFD dataset. The scaling 
factors are learned. The samples are generated by using the scaling factor from the 
first walkback step. Samples are sharper compared with Figure 
(\ref{fig:tfd-samples-no-sw}). This is also reflected by an improvement of 140 in 
Parzen-estimated log-likelihood.}
\label{fig:tfd-samples-yes-sw}
\end{figure*}

\section{Conclusion}
\svs{2}

We have introduced a new approach to training generative models, called
Generative Stochastic Networks (GSN), which includes generative denoising
auto-encoders as a special case (with no latent variable). It is an alternative to 
directly performing maximum likelihood on an explicit $P(X)$, with the objective of avoiding 
the intractable marginalizations and partition function that such direct likelihood methods often entail. The
training procedure is more similar to function approximation than to
unsupervised learning because the reconstruction distribution is simpler
than the data distribution, often unimodal (provably so in the limit of
very small noise).  This makes it possible to train unsupervised models that
capture the data-generating distribution simply using backprop and
gradient descent in a computational graph that includes noise injection.
The proposed theoretical results state that under mild conditions
(in particular that the noise injected in the networks prevents perfect reconstruction),
training a sufficient-capacity model to denoise and reconstruct
its observations (through a powerful family of reconstruction
distributions) suffices to capture the data-generating distribution through
a simple Markov chain. Another view is that we are training the
transition operator of a Markov chain whose stationary distribution
estimates the data distribution, which has the potential of corresponding
to an easier learning problem because the normalization constant for this
conditional distribution is generally dominated by fewer modes. These
theoretical results are extended to the case where the corruption is local
but still allows the chain to mix and to the case where some inputs are
missing or constrained (thus allowing to sample from a conditional
distribution on a subset of the observed variables or to learned structured
output models). The GSN framework is shown to lend to dependency networks a
valid estimator of the joint distribution of the observed variables even
when the learned conditionals are not consistent, also allowing to prove
in a new way the consistency of generalized pseudolikelihood training, associated with the
stationary distribution of a corresponding GSN (that randomly chooses a
subset of variables and then resamples it). Experiments have
been conducted to validate the theory, in the case where the GSN
architecture is a simple denoising auto-encoder and in the case
where the GSN emulates the Gibbs sampling process of a Deep Boltzmann
Machine. A quantitative evaluation of the samples
confirms that the training procedure works very well (in this case
allowing us to train a deep generative model without layerwise pretraining)
and can be used to perform conditional sampling of a subset of
variables given the rest. After early versions of this work
were published~\citep{Bengio+Laufer-arxiv-2013}, the GSN framework has been extended and
applied to classification problems in several different ways~\citep{Goodfellow-et-al-NIPS2013,Zhou+Troyanskaya-ICML2014,Zohrer+Pernkopf-NIPS2014-small}
yielding very interesting results. In addition to providing
a consistent generative interpretation to dependency networks, GSNs
have been used to provide one to Multi-Prediction Deep Boltzmann Machines~\citep{Goodfellow-et-al-NIPS2013}
and to provide a fast sampling algorithm for deep NADE~\citep{Yao-et-al-arXiv2014}.

\subsubsection*{Acknowledgements}

The authors would like to acknowledge the stimulating discussions and
help from Vincent Dumoulin, Aaron Courville, Ian Goodfellow, and Hod Lipson,
as well as funding from NSERC, CIFAR (YB is a CIFAR Senior Fellow), NASA (JY is a Space Technology Research Fellow),
and the Canada Research Chairs and Compute Canada.

{\small
\bibliography{strings,strings-shorter,ml,aigaion-shorter}
}

\section{Appendix: Argument for consistency based on local noise}

\appendix

\label{sec:local-consistency}

This section presents one direction that we pursed initially
to demonstrate that we had certain consistency properties
in terms of recovering the correct stationary distribution
when using a finite training sample.
We discuss this issue when we cite Theorem \ref{thm:schweitzer_inequality}
from the literature in section \ref{sec:from-DAE-to-GSN} and
thought it would be a good idea to
include our previous approach in this Appendix.

\vspace{1em}

The main theorem in \citet{Bengio-et-al-NIPS2013} (stated in supplemental as Theorem S1) requires that the
Markov chain be ergodic. A set of conditions
guaranteeing ergodicity is given in the aforementioned paper, but
these conditions are restrictive in requiring that ${\cal C}(\tilde{X}|X)>0$ everywhere
that ${\calP}(X)>0$. The effect of these restrictions is that
$P_\theta(X|\tilde{X})$ must have the capacity to model every mode of
${\calP}(X)$, exactly the difficulty we were trying to avoid. We show
here how we may also achieve the
required ergodicity through other means, allowing us to choose a
${\cal C}(\tilde{X}|X)$ that only makes small jumps, which in turn
only requires $P_\theta(X|\tilde{X})$ to model a small part of the
space around each $\tilde{X}$.

Let $P_{\theta_n}(X | \tilde{X})$
be a denoising auto-encoder that has been trained on $n$ training examples.
$P_{\theta_n}(X | \tilde{X})$
assigns a probability to $X$, given $\tilde{X}$,
when $\tilde{X} \sim {\cal C}(\tilde{X}|X)$.
This estimator defines a Markov chain $T_n$ obtained by sampling
alternatively an $\tilde{X}$ from ${\cal C}(\tilde{X}|X)$ and
an $X$ from $P_\theta(X | \tilde{X})$. Let $\pi_n$ be the asymptotic
distribution of the chain defined by $T_n$, if it exists.
The following theorem is proven by~\citet{Bengio-et-al-NIPS2013}.
\begin{customthm}{S1}
\label{thm:consistency}
{\bf If}
$P_{\theta_n}(X | \tilde{X})$ is a consistent estimator of the true conditional
distribution ${\calP}(X | \tilde{X})$
{\bf and}
$T_n$ defines an 
ergodic Markov chain,
{\bf then}
as $n\rightarrow \infty$, the asymptotic distribution $\pi_n(X)$ of the generated
samples converges to the data-generating distribution ${\calP}(X)$. 
\end{customthm}

In order for Theorem~\ref{thm:consistency} to apply, the chain must be ergodic. One set of conditions under which this occurs is given in the aforementioned paper. We slightly restate them here:

\begin{figure}[htpb]
\centering
\includegraphics[width=1\linewidth]{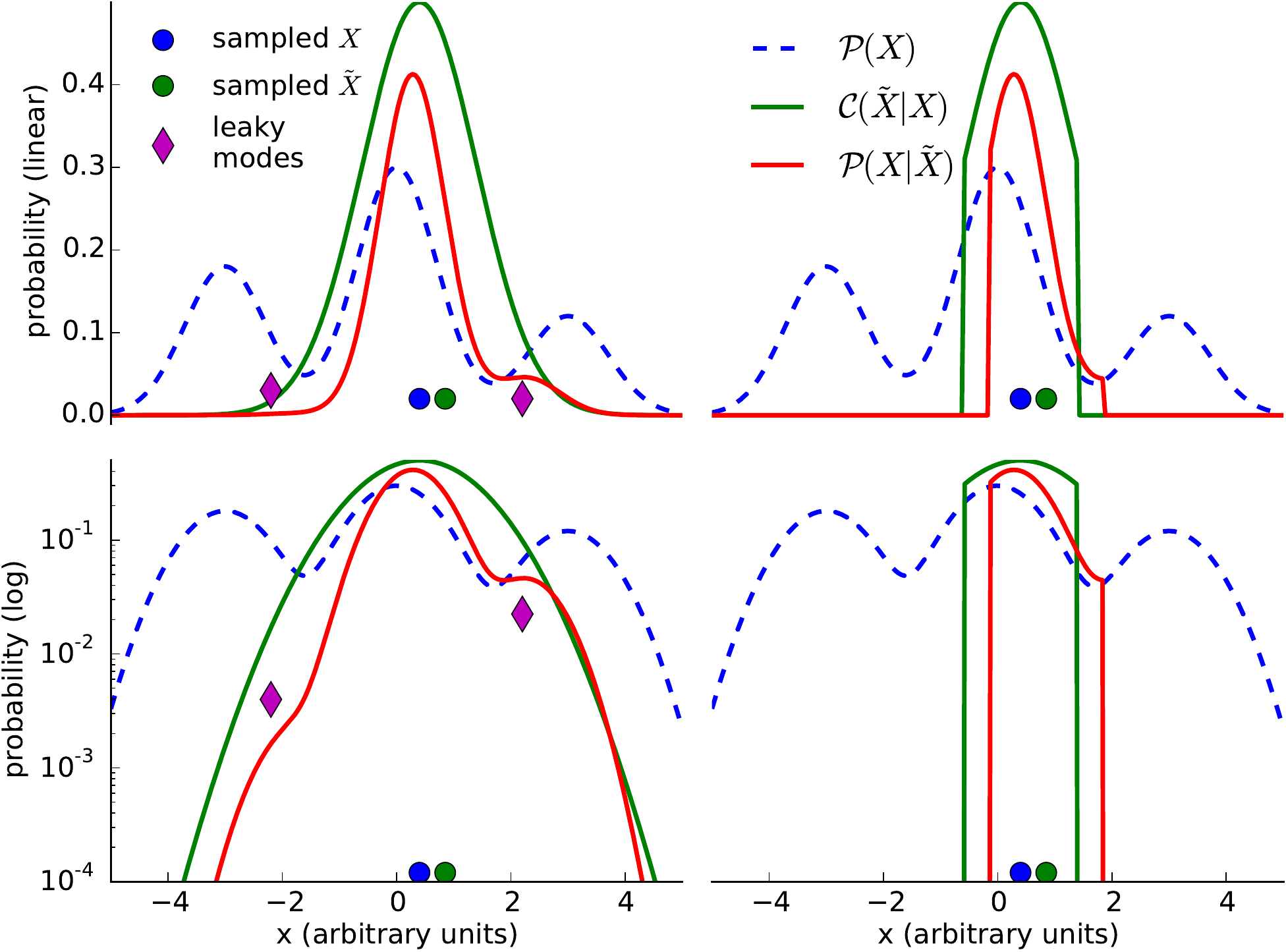}
\caption{If ${\cal C}(\tilde{X}|X)$ is globally supported as required by Corollary~\ref{cor:ergonew} \citep{Bengio-et-al-NIPS2013}, then for $P_{\theta_n}(X|\tilde{X})$ to converge to ${\calP}(X|\tilde{X})$, it will eventually have to model all of the modes in ${\calP}(X)$, even though the modes are damped (see ``leaky modes'' on the left). However, if we guarantee ergodicity through other means, as in Corollary~\ref{cor:ergonew2}, we can choose a local ${\cal C}(\tilde{X}|X)$ and allow $P_{\theta_n}(X|\tilde{X})$ to model only the local structure of ${\calP}(X)$ (see right).}
\label{fig:modes_1d}
\end{figure}

\begin{corollary}
\label{cor:ergonew}
{\bf If} the support for both the data-generating distribution and
denoising model are contained in and non-zero in
a finite-volume region $V$ (i.e., $\forall \tilde{X}$, $\forall X\notin V,\; {\calP}(X)=0, P_\theta(X|\tilde{X})=0$ and $\forall \tilde{X}$, $\forall X\in V,\; {\calP}(X)>0, P_\theta(X|\tilde{X})>0,  {\cal C}(\tilde{X}|X)>0$)
{\bf and} these statements remain
true in the limit of $n\rightarrow \infty$, {\bf then}
the chain defined by $T_n$ will be ergodic.
\end{corollary}

If conditions in Corollary~\ref{cor:ergonew} apply, then the chain will be ergodic and Theorem~\ref{thm:consistency} will apply. However, these conditions are sufficient, not necessary, and in many cases they may be artificially restrictive. In particular, Corollary~\ref{cor:ergonew} 
defines a large region $V$ containing any possible $X$ allowed by the model and requires that we maintain the probability of jumping between any two points in a single move to be greater than 0.
While this generous condition helps us easily guarantee the ergodicity of the chain, it also has the unfortunate side effect of requiring that, in order for $P_{\theta_n}(X|\tilde{X})$ to converge to the conditional distribution ${\calP}(X|\tilde{X})$, it must have the capacity to model every mode of ${\calP}(X)$, exactly the difficulty we were trying to avoid. The left two plots in Figure~\ref{fig:modes_1d} show this difficulty: because ${\cal C}(\tilde{X}|X)>0$ everywhere in $V$, every mode of $P(X)$ will leak, perhaps attenuated, into $P(X|\tilde{X})$.

Fortunately, we may seek ergodicity through other means. The following
corollary allows 
us to choose a ${\cal C}(\tilde{X}|X)$ that only makes small jumps, which in turn only requires $P_\theta(X|\tilde{X})$ to model a small part of the space $V$ around each $\tilde{X}$.

Let $P_{\theta_n}(X | \tilde{X})$
be a denoising auto-encoder that has been trained on $n$ training examples
and ${\cal C}(\tilde{X}|X)$ be some corruption distribution.
$P_{\theta_n}(X | \tilde{X})$
assigns a probability to $X$, given $\tilde{X}$,
when $\tilde{X} \sim {\cal C}(\tilde{X}|X)$ and $X \sim {\cal P}(X)$.
Define a Markov chain $T_n$ by alternately sampling
an $\tilde{X}$ from ${\cal C}(\tilde{X}|X)$ and
an $X$ from $P_\theta(X | \tilde{X})$.

\begin{corollary}
\label{cor:ergonew2}
{\bf If} the data-generating distribution is contained in and non-zero in
a finite-volume region $V$ (i.e., $\forall X\notin V,\; {\calP}(X)=0,$ and $\forall X\in V,\; {\calP}(X)>0$)
{\bf and} all pairs of points in $V$ can be connected by a finite-length path through $V$
{\bf and}
for some $\epsilon > 0$,
$\forall \tilde{X} \in V,\forall X\in V$ within $\epsilon$ of each other,
${\cal C}(\tilde{X}|X) > 0$ and $P_{\theta}(X|\tilde{X}) > 0$
{\bf and} these statements remain
true in the limit of $n\rightarrow \infty$,
{\bf then}
the chain defined by $T_n$ will be ergodic.
\end{corollary}

\begin{proof}
Consider any two points $X_a$ and $X_b$ in $V$.
By the assumptions of Corollary~\ref{cor:ergonew2}, there exists a finite length path between $X_a$ and $X_b$ through $V$. Pick one such finite length path $P$.
Chose a finite series of points $x = \{x_1, x_2, \ldots, x_k\}$ along $P$, with $x_1 = X_a$ and $x_k = X_b$ such that the distance between every pair of consecutive points $(x_i, x_{i+1})$ is less than $\epsilon$ as defined in Corollary~\ref{cor:ergonew2}.
Then the probability of sampling $\tilde{X} = x_{i+1}$ from ${\cal C}(\tilde{X}|x_i))$ will be positive, because ${\cal C}(\tilde{X}|X)) > 0$ for all $\tilde{X}$ within $\epsilon$ of $X$ by the assumptions of Corollary~\ref{cor:ergonew2}. Further, the probability of sampling $X = \tilde{X} = x_{i+1}$ from $P_{\theta}(X|\tilde{X})$ will be positive from the same assumption on $P$.
Thus the probability of jumping along the path from $x_i$ to $x_{i+1}$, $T_n(X_{t+1} = x_{i+1}|X_{t} = x_i)$, will be greater than zero for all jumps on the path.
Because there is a positive probability finite length path between all pairs of points in $V$, all states commute, and the chain is irreducible.
If we consider $X_a = X_b \in V$, by the same arguments $T_n(X_t = X_a|X_{t-1} = X_a) > 0$. Because there is a positive probability of remaining in the same state, the chain will be aperiodic.
Because the chain is irreducible and over a finite state space, it will be positive recurrent as well. Thus, the chain defined by $T_n$ is ergodic.
\end{proof}

Although this is a weaker condition that has the advantage of making
the denoising distribution even easier to model (probably having less
modes), we must be careful to choose the ball size $\epsilon$ large
enough to guarantee that one can jump often enough between the major modes of ${\calP}(X)$
when these are separated by zones of tiny probability. $\epsilon$ must be larger than half the 
largest distance one would have to travel across a desert of low probability
separating two nearby modes (which if not connected in this way would make
$V$ not anymore have a single connected component). Practically, there is
a trade-off between the difficulty of estimating ${\calP}(X|\tilde{X})$
and the ease of mixing between major modes separated by a very low density zone.

\end{document}